\documentclass[twoside]{article}
\usepackage[preprint]{aistats2027}
\usepackage[T1]{fontenc}
\usepackage{lmodern}
\usepackage{amsmath,amssymb,amsthm,mathtools}
\usepackage{graphicx,booktabs,microtype,xcolor}
\usepackage[round]{natbib}
\usepackage[hidelinks]{hyperref}
\newtheorem{theorem}{Theorem}
\newtheorem{lemma}[theorem]{Lemma}
\newtheorem{proposition}[theorem]{Proposition}
\newtheorem{corollary}[theorem]{Corollary}
\theoremstyle{definition}
\newtheorem{assumption}{Assumption}

\newcommand{\E}{\mathbb E}
\newcommand{\R}{\mathbb R}
\newcommand{\F}{\mathcal F}
\newcommand{\K}{\mathcal K}
\newcommand{\norm}[1]{\lVert #1\rVert}
\newcommand{\ip}[2]{\langle #1,#2\rangle}
\newcommand{\sk}{S_{\K}}
\newcommand{\Et}{\E_t}
\newcommand{\cL}{\mathcal L}
\newcommand{\amin}{\underline{\alpha}}
\newcommand{\bmax}{\overline{\beta}}
\newcommand{\ce}{c_\eta}            
\newcommand{\ca}{c_\alpha}          
\newcommand{\Cs}[1]{C^\star_{#1}}   
\newcommand{\hmax}{h_{\max}}        
\newcommand{\rK}{r_{\K}}            
\newcommand{\Reg}{\mathrm{Reg}}
\newcommand{\cQ}{\mathcal Q}        

\begin{document}
\runningtitle{Randomized Momentum Transport with Cheap Predictors}
\twocolumn[
\aistatstitle{Randomized Momentum Transport with Cheap Predictors}
\aistatsauthor{El Mahdi Chayti \And Martin Jaggi}
\aistatsaddress{Machine Learning and Optimization Laboratory (MLO), EPFL}
]

\begin{abstract}
Momentum methods such as STORM and second-order momentum track the gradient by adding a transport correction to momentum, at the cost of an extra oracle call per iteration. We introduce randomized momentum transport, which evaluates the correction with probability $q$ and otherwise uses a cheap prediction, with importance weighting that preserves the conditional mean and adds $(q^{-1}-1)$ times the prediction loss to the variance. Under transport-bias, second-moment, and residual-control assumptions, linear minimization oracle (LMO) updates, including normalized, sign, and exact spectral (Muon) directions, retain an $O(T^{-1/3})$ stationarity bound for each fixed $q>0$ with a single-gradient initialization. For unnormalized STORM, some $q<1$ lowers the leading coefficient of our sufficient gradient budget if and only if the gradient-difference noise modulus exceeds the smoothness constant, by at most about $10\%$. An online diagonal predictor admits a regret-to-convergence guarantee. At approximately equal oracle cost, querying a quarter of the corrections lowers the final training loss by $18$ to $43\%$ on MNIST1D and raises the test accuracy of Hessian-vector transport on CIFAR-10 by $1.5$ points.
\end{abstract}

\section{Introduction}
Momentum averages stochastic gradients across iterations. As the parameters move, this average must also track the change of the gradient. STORM adds the difference of two gradients evaluated on the same sample at consecutive iterates \citep{cutkosky2019storm}; second-order momentum adds a Hessian-vector product (HVP) \citep{salehkaleybar2022hcm,tran2022som}. Each correction costs one extra oracle call per iteration. We study how often the correction needs to be evaluated.

Let $C_t$ be the correction at iteration $t$, an estimate of the gradient change $\nabla f(x_t)-\nabla f(x_{t-1})$, and let $B_t$ be a cheap prediction of $C_t$. We use the randomized transport estimator
\begin{equation}\label{eq:intro-estimator}
\widehat C_t=B_t+\frac{J_t}{q}(C_t-B_t),
\quad J_t\sim\operatorname{Bernoulli}(q),
\end{equation}
with $J_t$ drawn independently of $(C_t,B_t)$. The correction is evaluated only when $J_t=1$. The estimator has the conditional mean of $C_t$, and its conditional variance exceeds that of $C_t$ by exactly $(q^{-1}-1)\E_t\norm{C_t-B_t}^2$, where $\E_t$ denotes expectation given the past. The prediction loss therefore sets the variance cost of querying less often.

\paragraph{Contributions.}
\begin{itemize}
\item \textbf{Randomized transport.} The estimator~\eqref{eq:intro-estimator} keeps the conditional mean of the correction and its covariance with the fresh gradient, and adds exactly $(q^{-1}-1)\Et\norm{C_t-B_t}^2$ to its variance (Lemma~\ref{lem:identities}).
\item \textbf{LMO convergence.} For gradient-difference and HVP transport with LMO directions, decreasing momentum weights, and a single-gradient initialization, the dual-norm stationarity measure of the LMO set is $O(T^{-1/3})$ after $T$ iterations for fixed $q>0$. The bound holds at every $T$ with explicit constants; $q$ enters only through $q^{-1}-1$ times the average prediction loss, and the initial error contributes $O(T^{-1})$ (Theorem~\ref{thm:lmo}, Corollary~\ref{cor:rate}).
\item \textbf{Gradient budget of randomized STORM.} For unnormalized STORM with a zero predictor, which uses $1+q$ gradients per iteration in expectation, we give an explicit finite-horizon gradient budget to reach expected gradient norm $\varepsilon$, with a leading term of order $\varepsilon^{-3}$. Some $q<1$ lowers this term if and only if the gradient-difference noise modulus $\delta$ exceeds the smoothness constant $L$, and by at most about $10\%$ (Theorem~\ref{thm:unnormalized}, Proposition~\ref{prop:ratio}).
\item \textbf{Learned predictor.} A diagonal predictor with $O(d)$ state, trained only on queried corrections, has expected regret $O(\sqrt{dT/q})$ against every fixed comparator. This gives a complete stationarity bound and oracle budget, and, when a fixed comparator has a uniformly smaller expected loss than the zero predictor, the learned bound improves on the zero-predictor bound once $q(T-1)$ exceeds a sufficient threshold linear in $d$ (Theorem~\ref{thm:regret}, Corollary~\ref{cor:learned-rate}).
\end{itemize}
The analysis extends to inexact LMO directions (Appendix~\ref{app:inexact}) and to MARS-style scaling of the whole correction or of its residual (Appendix~\ref{app:scaling}). On MNIST1D, randomized queries lower the final training loss relative to full correction at equal oracle cost, for both transports; on CIFAR-10, randomized HVP queries raise test accuracy; on a quadratic with controlled curvature noise, the benefit reverses as the curvature noise decreases, and the diagonal predictor then removes most of the reducible residual (Section~\ref{sec:experiments}).

Both corrections are matrix-free, and the predictor uses $O(d)$ arithmetic per iteration. The update geometry enters through the Euclidean radius of the LMO set; the cost of the direction oracle is accounted for separately.

\subsection{Related Work}
\paragraph{Recursive and switching estimators.}
STORM combines fresh gradients with recursive transport \citep{cutkosky2019storm}. PAGE switches with a fixed probability between a gradient refresh and a recursive correction \citep{li2021page}, and MARINA uses the same switching for compressed gradient differences in distributed learning \citep{gorbunov2021marina}. Our estimator draws a fresh gradient at every iteration and randomizes only the transport correction, which importance weighting keeps unbiased.

\paragraph{Shifted compression.}
The estimator~\eqref{eq:intro-estimator} applies an unbiased Bernoulli compressor to the residual of $C_t$ around the shift $B_t$. DIANA learns such shifts so that the compression variance vanishes over time \citep{mishchenko2019diana}. Here the shift is a predictor of a transport correction inside a single momentum recursion, and its accuracy is controlled through online regret.

\paragraph{Scaling and second-order correction.}
MARS scales recursive corrections \citep{yuan2025mars}, and \citet{shulgin2026mars} explain when scaling helps through a $\gamma$-similarity condition. Appendix~\ref{app:scaling} combines this scaling with randomized queries and analyzes scaling only the prediction residual. Hessian-vector transport underlies second-order momentum \citep{salehkaleybar2022hcm,tran2022som} and its LMO extensions \citep{khirirat2025lmo}. RanSOM randomizes step lengths so that Stein-type identities produce the correction \citep{chayti2026ransom}; our HVP variant keeps deterministic step lengths, randomizes the queries, and controls the Taylor remainder. Lazy Hessian methods reuse a Hessian over several iterations \citep{doikov2023lazy}, which is one source of cheap predictors.

\paragraph{Predictors and update geometry.}
Delay-compensated SGD uses diagonal transport approximations \citep{zheng2017delay}; Sophia reuses intermittently refreshed diagonal curvature for preconditioning \citep{liu2024sophia}; stochastic quasi-Newton methods learn curvature from spaced HVP evaluations \citep{byrd2016sqn}. We relate the accuracy of such predictions to the cost of randomized corrections and prove a regret guarantee for a diagonal predictor. LMO-based momentum methods include normalized momentum \citep{cutkosky2020normalized}, Scion \citep{pethick2025scion}, and the non-Euclidean trust-region view of Muon \citep{kovalev2025orthogonalization}. \citet{khirirat2025lmo} compare LMO-based methods with Polyak, extrapolated, second-order, and variance-reduced momentum, so an $O(T^{-1/3})$ rate with full correction is available in this setting; our contribution there is the randomized query and its residual-dependent constants. Inexact Muon directions are analyzed by \citet{shulgin2026inexact}.

\section{Randomized Transport}\label{sec:framework}
We minimize a differentiable function $f:\R^d\to\R$ over an integer horizon $T\ge1$. Matrix parameters are identified with a Euclidean space through the Frobenius inner product, and all unqualified norms are Euclidean. Fix $x_0$, let $f_*=\inf_x f(x)$, and set $\Delta=f(x_0)-f_*$. For $t\ge1$, write $s_t=x_t-x_{t-1}$ and $\Cs t=\nabla f(x_t)-\nabla f(x_{t-1})$.

\paragraph{Momentum recurrence.}
Choose a deterministic sequence $\alpha_t\in(0,1]$ with $\alpha_0=1$, and initialize $m_0=g_0$, a stochastic gradient at $x_0$. For $t\ge1$, with fresh stochastic gradient $g_t$ at $x_t$ and $\widehat C_t$ from~\eqref{eq:intro-estimator},
\begin{equation}\label{eq:momentum}
m_t=(1-\alpha_t)(m_{t-1}+\widehat C_t)+\alpha_t g_t.
\end{equation}
With $q=1$ this is the original transport recurrence. Our rate results use the decreasing schedule with floor $\amin$,
\begin{equation}\label{eq:alpha-schedule}
\alpha_t=\max\{\amin,2/(t+2)\},\qquad 0<\amin\le1.
\end{equation}

\paragraph{Directions.}
Let $\K$ be compact, convex, symmetric about the origin ($\K=-\K$), and full dimensional, with radius $R=\sup_{u\in\K}\norm{u}$ and inradius $\rK=\sup\{r:\ \norm{u}\le r\Rightarrow u\in\K\}>0$. Its support function $\sk(z)=\max_{u\in\K}\ip{z}{u}$ is a norm, and $\sk(z)\ge\rK\norm{z}$. Choose a measurable LMO direction and update
\begin{equation}\label{eq:lmo-update}
u_t\in\arg\max_{u\in\K}\ip{m_t}{u},\qquad
x_{t+1}=x_t-\eta u_t,
\end{equation}
so that $\norm{s_t}\le\eta R$ for every $t\ge1$. At $m_t=0$ choose $u_t=0$.

\begin{table}[t]
\caption{Direction sets and stationarity measures. Each unit ball has inradius $\rK=1$. The spectral row concerns matrix parameters in $\R^{d_1\times d_2}$.}
\label{tab:geometry}
\centering
\small
\setlength{\tabcolsep}{4pt}
\begin{tabular}{llll}
\toprule
Unit ball of & Direction & $\sk(\nabla f)$ & $R$\\
\midrule
$\ell_2$ & $m/\norm{m}_2$ & $\norm{\nabla f}_2$ & $1$\\
$\ell_\infty$ & $\operatorname{sign}(m)$ & $\norm{\nabla f}_1$ & $\sqrt d$\\
Spectral & Polar factor & $\norm{\nabla f}_{\rm nuc}$ & $\sqrt{\min\{d_1,d_2\}}$\\
\bottomrule
\end{tabular}
\end{table}

\begin{figure}[t]
\hrule\smallskip
\textbf{Algorithm 1: Randomized momentum transport}
\smallskip
\textbf{Input:} $x_0$, $T$, $\eta$, $q$, $\{\alpha_t\}_{t=0}^{T-1}$, predictor state.\\
Compute one stochastic gradient $m_0=g_0$ at $x_0$.\\
\textbf{For} $t=0,\ldots,T-1$:\\
\hspace*{1em}Choose $u_t$ by the LMO; set $x_{t+1}=x_t-\eta u_t$.\\
\hspace*{1em}\textbf{If} $t=T-1$, stop.\\
\hspace*{1em}Compute fresh $g_{t+1}$ and prediction $B_{t+1}$.\\
\hspace*{1em}Draw independent $J_{t+1}\sim\operatorname{Bernoulli}(q)$.\\
\hspace*{1em}\textbf{If} $J_{t+1}=1$, evaluate $C_{t+1}$.\\
\hspace*{1em}Form $\widehat C_{t+1}$ by~\eqref{eq:intro-estimator}.\\
\hspace*{1em}Update $m_{t+1}$ by~\eqref{eq:momentum}, using $\alpha_{t+1}$.\\
\hspace*{1em}Train future predictions if a correction was observed.\\
\textbf{Return:} $x_U$, with independent $U\sim\mathrm{Unif}\{0,\ldots,T-1\}$.
\smallskip\hrule
\end{figure}

\paragraph{History and assumptions.}
Let $\F_{-1}$ be trivial and, for $t\ge0$, let $\F_t$ be the $\sigma$-algebra generated by all oracle samples, query decisions, and evaluated quantities used to form $m_0,\ldots,m_t$ and the predictor state used at iteration $t+1$. Then $x_{t+1}$, $s_{t+1}$, and $\Cs{t+1}$ are $\F_t$-measurable. Write $\Et[\cdot]=\E[\cdot\mid\F_{t-1}]$. The correction $C_t$ is treated as a potential oracle value, defined whether or not it is evaluated. Relevant second moments are assumed finite.

\begin{assumption}[Objective and gradients]\label{ass:objective}
$f_*>-\infty$ and $f$ is $L$-smooth. For every $t\ge0$, $\Et g_t=\nabla f(x_t)$ and $\Et\norm{g_t-\nabla f(x_t)}^2\le\sigma^2$.
\end{assumption}

\begin{assumption}[Transport]\label{ass:transport}
Let $\mu_t=\Et C_t$. For constants $b,v\ge0$ and every $t\ge1$,
\begin{align}
\norm{\mu_t-\Cs t}&\le b\norm{s_t}^2,\label{eq:bias-assumption}\\
\Et\norm{C_t-\mu_t}^2&\le v^2\norm{s_t}^2.\label{eq:variance-assumption}
\end{align}
The gradient $g_t$ and the correction $C_t$ may share oracle randomness.
\end{assumption}

\begin{assumption}[Query model]\label{ass:query}
The probability $q\in(0,1]$ is fixed. For every $t\ge1$, $J_t\mid\F_{t-1}\sim\operatorname{Bernoulli}(q)$, and $J_t$ is conditionally independent of $(C_t,B_t,g_t)$ given $\F_{t-1}$.
\end{assumption}
The predictor $B_t$ may use the past and cheap information from the current sample, but not the value $C_t$ that $J_t$ decides whether to query.

\subsection{Two Transport Oracles}
\paragraph{Gradient differences.}
Let $g(x;\xi)$ be unbiased for $\nabla f(x)$ at every $x$. For a fresh sample $\xi_t$ independent of $\F_{t-1}$, set $g_t=g(x_t;\xi_t)$ and
\begin{equation}\label{eq:storm-correction}
C_t=g(x_t;\xi_t)-g(x_{t-1};\xi_t).
\end{equation}
This oracle is unbiased for $\Cs t$. Assume the centered similarity condition
\begin{equation}\label{eq:similarity}
\E\norm{g(x;\xi)-g(y;\xi)-[\nabla f(x)-\nabla f(y)]}^2
\le\delta^2\norm{x-y}^2.
\end{equation}
Then $b=0$ and $v=\delta$. A queried correction costs one extra gradient, at $x_{t-1}$ on the current sample $\xi_t$.

\paragraph{Hessian-vector products.}
Suppose $f$ is twice differentiable with a $\rho$-Lipschitz Hessian. For a Hessian oracle $H_t$ with $\Et H_t=\nabla^2f(x_{t-1})$, assume
\begin{equation}\label{eq:hvp-noise}
\Et\norm{(H_t-\nabla^2 f(x_{t-1}))s_t}^2
\le\tau^2\norm{s_t}^2,
\end{equation}
and set $C_t=H_ts_t$. Taylor's theorem gives $b=\rho/2$, and~\eqref{eq:hvp-noise} gives $v=\tau$. An HVP at $x_t$ satisfies the same bounds when it is unbiased at $x_t$ and satisfies~\eqref{eq:hvp-noise} with $x_t$ in place of $x_{t-1}$; it can reuse the derivative graph of the current gradient. Both endpoints, and gradient differences, satisfy $\norm{\mu_t}\le L\norm{s_t}$.

\subsection{Oracle Accounting}
Algorithm~1 uses $T$ gradients and $Q_T\sim\operatorname{Binomial}(T-1,q)$ correction queries. For gradient differences,
\begin{equation}\label{eq:count}
N_T=T+Q_T,\qquad \E N_T=(1+q)T-q.
\end{equation}
For HVP transport, gradient and HVP counts remain separate. If an HVP costs $c_H$ gradient-equivalent units, the expected oracle work is $T+c_Hq(T-1)$. Predictor, direction, and systems overhead are counted separately in runtime comparisons.

\section{Residual-Dependent Convergence}\label{sec:convergence}
The analysis separates the conditional variance of the correction, momentum tracking, and descent. Write $\zeta_t=g_t-\nabla f(x_t)$.

\begin{lemma}[Transport identities]\label{lem:identities}
Under Assumptions~\ref{ass:transport} and~\ref{ass:query}, for every $t\ge1$,
\begin{align}
\Et\widehat C_t&=\mu_t,\label{eq:mean-id}\\
\Et\norm{\widehat C_t-\mu_t}^2
&=\Et\norm{C_t-\mu_t}^2\notag\\
&\quad +(q^{-1}-1)\Et\norm{C_t-B_t}^2,\label{eq:variance-id}\\
\Et\ip{\zeta_t}{\widehat C_t-\mu_t}
&=\Et\ip{\zeta_t}{C_t-\mu_t}.\label{eq:covariance-id}
\end{align}
The same second-moment increment holds for centering at any $\F_{t-1}$-measurable point, in particular at $\Cs t$.
\end{lemma}
\begin{proof}
Write $\widehat C_t=C_t+(J_t/q-1)(C_t-B_t)$. Conditional on $(\F_{t-1},C_t,B_t,g_t)$, the added term has zero mean and second moment $(q^{-1}-1)\norm{C_t-B_t}^2$. Expanding gives the three identities.
\end{proof}

Define the average prediction loss
\begin{equation}\label{eq:prediction-loss}
\cL_T=\frac1T\sum_{t=1}^{T-1}\E\norm{C_t-B_t}^2,
\end{equation}
normalized by the number $T$ of iterates so that it enters the averaged bounds below directly. For schedule~\eqref{eq:alpha-schedule}, write $\bmax=1-\amin$ and define
\begin{align}
W_T={}&\sigma^2\amin+\sigma v\eta R\notag\\
&+\frac{\bmax^2}{\amin(2-\amin)}
\left[v^2\eta^2R^2+(q^{-1}-1)\cL_T\right].\label{eq:WT}
\end{align}
The covariance term involves the original noise modulus $v$; only the prediction loss is multiplied by $q^{-1}-1$.

\begin{theorem}[LMO convergence with decreasing momentum weights]\label{thm:lmo}
Under Assumptions~\ref{ass:objective}--\ref{ass:query}, Algorithm~1 with $\eta>0$ and schedule~\eqref{eq:alpha-schedule} satisfies
\begin{align}
\E\sk(\nabla f(x_U))\le{}&\frac{\Delta}{\eta T}+\frac{L\eta R^2}{2}
+\frac{4R\sigma}{T}+\frac{4\sqrt2R\sigma}{\sqrt T}\notag\\
&+\frac{2\bmax b\eta^2R^3}{\amin}+2R\sqrt{W_T}.\label{eq:lmo-bound}
\end{align}
\end{theorem}
The terms $4R\sigma/T$ and $4\sqrt2R\sigma/\sqrt T$ come from the initial momentum error and from the fresh-gradient noise during the decreasing part of the schedule, respectively.

\paragraph{Proof outline.}
With $e_t=m_t-\nabla f(x_t)$ and $\beta_t=1-\alpha_t$,
\begin{equation}\label{eq:error-recursion}
e_t=\beta_t e_{t-1}+\beta_t r_t+w_t,
\quad r_t=\mu_t-\Cs t,
\end{equation}
where $w_t=\beta_t(\widehat C_t-\mu_t)+\alpha_t\zeta_t$ is a martingale difference. Unrolling gives a weighted initial error, a drift sum, and a martingale sum. The schedule yields
$\prod_{i=1}^t\beta_i\le2/[(t+1)(t+2)]$, so the cumulative initial-error weight is at most two. The remaining terms follow from weighted second-moment bounds and the alignment inequality
\begin{equation}\label{eq:lmo-alignment}
\ip{\nabla f(x_t)}{u_t}
\ge\sk(\nabla f(x_t))-2R\norm{e_t}.
\end{equation}
Appendix~\ref{app:lmo} gives the proof and the corresponding bound for arbitrary deterministic $\alpha_t$.

\begin{corollary}[Complete stationarity bound]\label{cor:rate}
Let $\ce>0$ and $0<\ca\le1$, and suppose $\cL_T\le\bar\kappa^2\eta^2R^2$ for some $\bar\kappa\ge0$. Set $\eta=\ce T^{-2/3}$ and $\amin=\ca T^{-2/3}$ in~\eqref{eq:alpha-schedule}. For $\kappa\ge0$ define
\begin{align}
Z_q(\kappa)&=\ca\sigma^2+\sigma v\ce R\notag\\
&\quad+\frac{\ce^2R^2}{\ca}\left[v^2+(q^{-1}-1)\kappa^2\right],\label{eq:Zq}\\
\mathsf A_q(\kappa)&=\Delta/\ce+2R\sqrt{Z_q(\kappa)},\notag\\
\mathsf U&=4\sqrt2R\sigma,\qquad \mathsf I=4R\sigma,\notag\\
\mathsf D&=L\ce R^2/2+2b\ce^2R^3/\ca.\label{eq:rate-constants}
\end{align}
Then, under the assumptions of Theorem~\ref{thm:lmo}, for every $T\ge1$,
\begin{align}
\E\sk(\nabla f(x_U))\le{}&\mathsf A_q(\bar\kappa)T^{-1/3}
+\mathsf U T^{-1/2}\notag\\
&+\mathsf D T^{-2/3}+\mathsf I T^{-1}.\label{eq:rate}
\end{align}
\end{corollary}
The four terms collect the descent and tracking error, the noise of the decreasing phase, smoothness and Taylor drift, and the initial momentum error. For fixed $q>0$ and a uniform $\bar\kappa$, the bound is $O(T^{-1/3})$.

For $B_t=0$ and $\norm{\mu_t}\le M\norm{s_t}$, Assumption~\ref{ass:transport} gives $\Et\norm{C_t}^2\le(v^2+M^2)\norm{s_t}^2$, so Corollary~\ref{cor:rate} applies with $\bar\kappa^2=v^2+M^2$. Both transport oracles satisfy this with $M=L$. In general, prediction accuracy enters only through $v^2+(q^{-1}-1)\bar\kappa^2$.

\paragraph{Update geometry.}
Table~\ref{tab:geometry} lists $R$ for the three direction sets; the factor $R$ converts Euclidean tracking error into the stationarity measure $\sk$. Smoothness and noise are measured in the Euclidean norm, so the geometry enters only through $R$ and $\rK$. For sign directions the alignment multiplier is $R=\sqrt d$, and the step-dependent terms contain further powers of $R$; analyses that measure smoothness in the norm associated with $\K$ \citep{pethick2025scion,kovalev2025orthogonalization,khirirat2025lmo} can have a sharper dimension dependence.

\paragraph{Inexact directions.}
Appendix~\ref{app:inexact} covers two inexactness models. If $u_t\in\K$ and $\ip{m_t}{u_t}\ge\sk(m_t)-\epsilon_t$, then~\eqref{eq:lmo-bound} holds with the additional term $T^{-1}\sum_t\E\epsilon_t$. If instead $\norm{u_t-u_t^\circ}\le\epsilon_{\rm L}$ for an exact LMO direction $u_t^\circ$ of $m_t$, an additive model as in \citet{shulgin2026inexact} measured in the Euclidean norm, and $\epsilon_{\rm L}<\rK$, then $(1-\epsilon_{\rm L}/\rK)\,\E\sk(\nabla f(x_U))$ is bounded by the right side of~\eqref{eq:lmo-bound} with $R$ replaced by $R+\epsilon_{\rm L}$.

\section{Gradient Complexity}\label{sec:cost}
We specialize to gradient-difference transport~\eqref{eq:storm-correction} with $B_t=0$ and the unnormalized step $x_{t+1}=x_t-\eta m_t$. Define
\begin{equation}\label{eq:Vq}
V_q^2=\frac{\delta^2+(1-q)L^2}{q}.
\end{equation}
The centered randomized transport has variance at most $V_q^2\norm{s_t}^2$, while its covariance with the fresh gradient noise is bounded by $\sigma\delta\norm{s_t}$ (Appendix~\ref{app:transport}).

\begin{theorem}[Unnormalized randomized STORM]\label{thm:unnormalized}
Let Assumptions~\ref{ass:objective} and~\ref{ass:query} hold, with the gradient-difference correction~\eqref{eq:storm-correction}, similarity~\eqref{eq:similarity}, $B_t=0$, and $x_{t+1}=x_t-\eta m_t$. Let $c>0$, suppose $c\eta^2\le1$ and
\begin{equation}\label{eq:slack}
\theta_\eta=\frac{1-L\eta}{2}-\frac{V_q^2}{c}>0,
\end{equation}
and use schedule~\eqref{eq:alpha-schedule} with $\amin=c\eta^2$. Then
\begin{align}
\E\norm{\nabla f(x_U)}^2\le{}&\frac{2\Delta}{\eta T}
+2\sigma^2\eta^2\left(c+\frac{\delta^2}{\theta_\eta}\right)\notag\\
&+\frac{\sigma^2}{T}\left[2+4\log\frac{T+2}{2}\right].\label{eq:unnormalized}
\end{align}
The gradient count is~\eqref{eq:count}.
\end{theorem}

The proof uses the potential $\E[f(x_t)-f_*]+(\eta/\alpha_t)\E\norm{e_t}^2$, whose initial value is at most $\Delta+\eta\sigma^2$. The covariance term is kept before completing the square in $\norm{m_t}$ (Appendix~\ref{app:unnormalized}).

Assume $\Delta,\sigma,\delta>0$ and fix $q$. Minimizing the coefficient of $\eta^2$ at $L=0$ over $c$ gives
\begin{equation}\label{eq:optimal-c}
c_q=2V_q(V_q+\delta),\qquad
K_q=4\sigma^2(V_q+\delta)^2.
\end{equation}
Use $c=c_q$ and define
\begin{align}
\eta_{\max}&=\min\left\{\frac1{\sqrt{c_q}},
\frac{\delta}{2L(V_q+\delta)}\right\},\notag\\
T_{\rm cap}&=\max\left\{1,\left\lceil\frac{\Delta}{K_q\eta_{\max}^3}\right\rceil\right\},\label{eq:cap-threshold}
\end{align}
omitting the second term in $\eta_{\max}$ when $L=0$. For $T\ge T_{\rm cap}$, choose $\eta=(\Delta/(K_qT))^{1/3}$. Then
\begin{align}
\E\norm{\nabla f(x_U)}^2\le{}&3(\Delta^2K_q)^{1/3}T^{-2/3}\notag\\
&+\frac{2L\Delta+2\sigma^2+4\sigma^2\log((T+2)/2)}{T}.\label{eq:leading-rate}
\end{align}
Appendix~\ref{app:unnormalized} derives~\eqref{eq:leading-rate} and the bound for $T<T_{\rm cap}$.

For a target $\varepsilon>0$, define
\begin{align}
T_0^{\rm S}&=\max\{40,T_{\rm cap},3\sqrt3\Delta\sqrt{K_q}\varepsilon^{-3}\},\notag\\
\varphi_0&=\frac{2L\Delta+2\sigma^2+4\sigma^2\log((T_0^{\rm S}+2)/2)}{T_0^{\rm S}},\notag\\
T_\varepsilon^{\rm S}&=\left\lceil T_0^{\rm S}(1+\varphi_0/\varepsilon^2)^{3/2}\right\rceil.\label{eq:storm-horizon}
\end{align}
Running $T_\varepsilon^{\rm S}$ iterations guarantees $\E\norm{\nabla f(x_U)}\le\varepsilon$ with expected gradient budget $(1+q)T_\varepsilon^{\rm S}-q$, whose leading term is
\begin{equation}\label{eq:leading-budget}
N_q^{\rm lead}(\varepsilon)
=6\sqrt3\,\Delta\sigma(1+q)(V_q+\delta)\varepsilon^{-3}.
\end{equation}

\begin{proposition}[Leading budget comparison]\label{prop:ratio}
Let $\omega=L^2/\delta^2$. The ratio of the leading coefficients in~\eqref{eq:leading-budget} at $q$ and at $1$ is
\begin{equation}\label{eq:ratio}
\cQ(q;\omega)=\frac{1+q}{4}
\left(1+\sqrt{\frac{1+\omega}{q}-\omega}\right).
\end{equation}
There exists $q\in(0,1)$ with $\cQ(q;\omega)<1$ if and only if $\omega<1$.
\end{proposition}
The comparison concerns leading sufficient budgets of the same analysis, with $c$ optimized separately at each $q$. As $\omega\to0$, the minimizer approaches $q\approx0.4320$ and the ratio approaches $0.9027$; at $q=1/2$ the limit is $0.9053$. The reduction of the leading bound is therefore at most about $9.7\%$.

\begin{figure}[t]
\centering
\includegraphics[width=0.6\columnwidth]{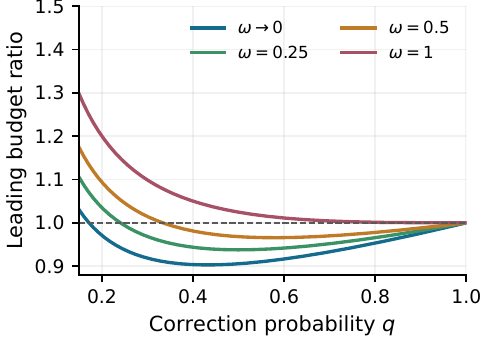}
\caption{Leading gradient-budget ratio~\eqref{eq:ratio}. Lower is better, and the dashed line is full correction. For $\omega<1$ the minimum lies below one; the $\omega=1$ curve is minimized at $q=1$.}
\label{fig:ratio}
\end{figure}

\paragraph{Interpretation.}
Skipping a query randomizes both the centered part of the correction, which already contributes tracking error, and its mean $\Cs t$, which produces the term $(1-q)L^2/q$ in $V_q^2$. A predictable surrogate $B_t$ can remove part of this mean contribution (Section~\ref{sec:predictors}). If one correction costs $c_C$ gradient-equivalent units, the factor $1+q$ becomes $1+c_Cq$ before the comparison is optimized.

\section{Learning Cheap Predictors}\label{sec:predictors}
Theorem~\ref{thm:lmo} charges randomization through $(q^{-1}-1)\cL_T$. This section asks how much of $\cL_T$ a cheap predictor can remove, gives an online predictor that removes part of it at an explicit learning cost, and determines when the learning cost is repaid.

\subsection{What a Predictor Can Remove}
For predictable $B_t$, orthogonality gives
\begin{equation}\label{eq:irreducible}
\Et\norm{C_t-B_t}^2
=\Et\norm{C_t-\mu_t}^2+\norm{\mu_t-B_t}^2.
\end{equation}
The first term is the conditional oracle variance, which no predictor changes; a predictor can only reduce the mean mismatch $\norm{\mu_t-B_t}^2$. Assume from now on that $\norm{\mu_t}\le M\norm{s_t}$, which both transport oracles satisfy with $M=L$, and put $V_C^2=v^2+M^2$. The zero predictor then has $\cL_T\le\eta^2R^2V_C^2$; it is the baseline of this section.

For both oracles the mean correction is a curvature action on the displacement: $\mu_t=\nabla^2f(x_{t-1})s_t$ for HVP transport at $x_{t-1}$, and $\mu_t=\int_0^1\nabla^2f(x_{t-1}+\vartheta s_t)s_t\,d\vartheta$ for gradient differences when $f$ is twice differentiable. This suggests linear predictors $B_t=A_ts_t$ with predictable $A_t$. With $\norm{A_t}_{\mathrm{op}}$ bounded, their residual scales with $\norm{s_t}$ and vanishes when the parameters do not move. A known quadratic component is the exact case: if $f=\widetilde f+(\lambda_0/2)\norm{x}^2$, then $A_t=\lambda_0 I$ leaves only the transport of $\widetilde f$ in the residual. Curvature cached from earlier iterations, as in lazy Hessian methods \citep{doikov2023lazy}, and scalar secant estimates are inexact cases. Lemma~\ref{lem:identities} also allows predictors that use cheap current-sample information, to which~\eqref{eq:irreducible} does not apply. The cheapest learnable choice is diagonal, $A_t=\operatorname{diag}(h_t)$, analyzed next.

\subsection{An Online Diagonal Predictor}
Let $\hmax>0$. For a predictable vector $h_t\in[-\hmax,\hmax]^d$, set
\begin{equation}\label{eq:diagonal-predictor}
B_t=h_t\odot s_t,\qquad h_1=0.
\end{equation}
When $s_t\ne0$, define $z_t=s_t/\norm{s_t}$ and $y_t=C_t/\norm{s_t}$; set both to zero when $s_t=0$, in which case Assumption~\ref{ass:transport} gives $C_t=0$ almost surely. The prediction loss
\begin{equation}\label{eq:online-loss}
\ell_t(h)=\norm{h\odot z_t-y_t}^2
\end{equation}
is convex in $h$. Only queried corrections reveal $y_t$. After forming $m_t$, update the predictor with step size $\eta_h>0$:
\begin{align}
\widehat a_t&=\frac{2J_t}{q}(h_t\odot z_t-y_t)\odot z_t,\notag\\
h_{t+1}&=\Pi_{[-\hmax,\hmax]^d}(h_t-\eta_h\widehat a_t).\label{eq:online-update}
\end{align}
On unqueried steps $\widehat a_t=0$, so no correction is evaluated. The predictor uses $O(d)$ memory and arithmetic and no oracle queries beyond the corrections already evaluated, and its coefficients may be negative.

\begin{theorem}[Expected fixed-comparator regret]\label{thm:regret}
Under Assumptions~\ref{ass:transport} and~\ref{ass:query}, the LMO update~\eqref{eq:lmo-update}, and $\norm{\mu_t}\le M\norm{s_t}$, let $n=T-1\ge1$. For every fixed $h\in[-\hmax,\hmax]^d$,
\begin{align}
\sum_{t=1}^n\E[\ell_t(h_t)-\ell_t(h)]
\le \frac{d\hmax^2}{2\eta_h}
+\frac{4\eta_h n(\hmax^2+V_C^2)}q.\label{eq:regret}
\end{align}
The choice $\eta_h=\hmax\sqrt{dq/[8n(\hmax^2+V_C^2)]}$ makes the right side equal to
\begin{equation}\label{eq:regret-value}
\Reg_n=\hmax\sqrt{8dn(\hmax^2+V_C^2)/q}.
\end{equation}
Consequently,
\begin{align}
\cL_T\le\frac{\eta^2R^2}{T}
\left[\inf_{h\in[-\hmax,\hmax]^d}\sum_{t=1}^n\E\ell_t(h)+\Reg_n\right].\label{eq:online-conversion}
\end{align}
\end{theorem}

Define the average loss of the best fixed comparator and the regret coefficient by
\begin{align}
\kappa_T^2&=\frac1{T-1}\inf_{h\in[-\hmax,\hmax]^d}
\sum_{t=1}^{T-1}\E\ell_t(h),\label{eq:comparator-loss}\\
\Gamma_q&=\hmax\sqrt{8d(\hmax^2+V_C^2)/q}.\label{eq:Gamma}
\end{align}
The comparator is fixed across iterations, its loss is evaluated on the learner's trajectory, and the infimum is outside the expectations. Since $\Reg_{T-1}=\Gamma_q\sqrt{T-1}$,
\begin{align}
\cL_T
&\le\eta^2R^2\left[\frac{T-1}{T}\kappa_T^2
+\frac{\Gamma_q\sqrt{T-1}}T\right]\notag\\
&\le\eta^2R^2\left[\kappa_T^2+\frac{\Gamma_q}{\sqrt T}\right].\label{eq:explicit-LT}
\end{align}
The zero comparator gives $\kappa_T^2\le V_C^2$, so~\eqref{eq:explicit-LT} holds without an approximation assumption and is never worse than the zero-predictor bound by more than the regret term. Because $\kappa_T$ is evaluated along the learner's own trajectory, which depends on $\eta$, $\amin$, and $\eta_h$, certifying an improvement before the run requires a bound $\kappa_T\le\kappa<V_C$ known in advance; without structural information about the problem, the available choice is $\kappa=V_C$.

\paragraph{When learning pays off.}
The bound~\eqref{eq:explicit-LT} trades a smaller comparator loss against the regret term $\Gamma_q/\sqrt{T}$, which dominates for short horizons. Suppose $\kappa_T\le\kappa<V_C$ and write $n=T-1$. The first line of~\eqref{eq:explicit-LT} gives
\begin{equation}\label{eq:prediction-threshold}
\cL_T\le\eta^2R^2\frac nT
\left(\kappa^2+\frac{\Gamma_q}{\sqrt n}\right),
\end{equation}
while the zero predictor gives $\cL_T\le\eta^2R^2(n/T)V_C^2$. The learned bound is strictly smaller whenever
\begin{equation}\label{eq:break-even}
T-1>\frac{\Gamma_q^2}{(V_C^2-\kappa^2)^2}.
\end{equation}
Both runs use the same oracle counts at equal $T,q,\eta,\amin$, and for $q<1$ and $\amin<1$ substitution in~\eqref{eq:WT} gives a strictly smaller stationarity certificate. In expected queried corrections, the condition reads
\begin{equation}\label{eq:query-threshold}
q(T-1)>\frac{8d\hmax^2(\hmax^2+V_C^2)}{(V_C^2-\kappa^2)^2}.
\end{equation}
This sufficient threshold is linear in $d$ for a fixed predictor radius, oracle constants, and comparator gap; it bounds the learning cost from above and is not a necessary number of queries.

\subsection{Convergence and Oracle Complexity}
\begin{corollary}[Complete learned-predictor bound]\label{cor:learned-rate}
Under the assumptions of Theorems~\ref{thm:lmo} and~\ref{thm:regret}, with $\eta_h$ as in Theorem~\ref{thm:regret}, let $\eta=\ce T^{-2/3}$ and $\amin=\ca T^{-2/3}$ with $\ce>0$ and $0<\ca\le1$. Define
\begin{equation}\label{eq:Gq}
\mathsf G_q=2\ce R^2\sqrt{(q^{-1}-1)\Gamma_q/\ca}.
\end{equation}
For $T\ge2$,
\begin{align}
\E\sk(\nabla f(x_U))\le{}&\mathsf A_q(\kappa_T)T^{-1/3}
+\mathsf U T^{-1/2}\notag\\
&+\mathsf G_qT^{-7/12}
+\mathsf D T^{-2/3}+\mathsf I T^{-1},\label{eq:learned-rate}
\end{align}
with the coefficients of~\eqref{eq:rate-constants}.
\end{corollary}
The regret part of~\eqref{eq:explicit-LT} contributes $[\ce^2R^2(q^{-1}-1)\Gamma_q/\ca]T^{-7/6}$ to $W_T$, and its square root gives the term $\mathsf G_qT^{-7/12}$ (Appendix~\ref{app:regret}). At $q=1$, $\mathsf G_q=0$ and the predictor cancels from the estimator. Since $Z_q(\kappa_T)\ge\ca\sigma^2$, for $\sigma>0$ the inequality $\sqrt{x+y}-\sqrt{x}\le y/(2\sqrt x)$ gives the alternative regret term
\begin{equation}\label{eq:sharp-regret}
\frac{\ce^2R^3(q^{-1}-1)\Gamma_q}{\ca\sqrt{Z_q(\kappa_T)}}T^{-5/6},
\end{equation}
and the minimum of the two terms may be used.

Assume a bound $\kappa_T\le\kappa$ that holds for every horizon under consideration; $\kappa=V_C$ always qualifies. Write $\mathsf A=\mathsf A_q(\kappa)$ and
\begin{equation}\label{eq:learned-remainder}
\mathcal E(t)=\mathsf U t^{-1/2}+\mathsf G_qt^{-7/12}
+\mathsf D t^{-2/3}+\mathsf I t^{-1}.
\end{equation}
For a target $\varepsilon>0$, define
\begin{align}
T_0&=\max\{2,(\mathsf A/\varepsilon)^3\},\notag\\
T_\varepsilon&=\left\lceil T_0
\left(1+\frac{\mathcal E(T_0)}{\varepsilon}\right)^3\right\rceil.\label{eq:learned-horizon}
\end{align}
Running the algorithm for the horizon $T_\varepsilon$, with the corresponding $\eta$, $\amin$, and $\eta_h$, guarantees $\E\sk(\nabla f(x_U))\le\varepsilon$. For gradient-difference transport, the expected gradient count is
\begin{equation}\label{eq:learned-budget}
\E N_{\rm grad}=(1+q)T_\varepsilon-q.
\end{equation}
HVP transport uses $T_\varepsilon$ fresh gradients and $q(T_\varepsilon-1)$ HVPs in expectation, for expected oracle work $(1+c_Hq)T_\varepsilon-c_Hq$. The predictor adds $O(dT_\varepsilon)$ arithmetic and $O(d)$ state.

Sufficiency of~\eqref{eq:learned-horizon} follows because every power in $\mathcal E(t)$ decays at least as fast as $t^{-1/3}$, so for $t\ge T_0$ the bound is at most $(T_0/t)^{1/3}[\varepsilon+\mathcal E(T_0)]$. The same construction applies with the alternative term~\eqref{eq:sharp-regret} evaluated at $\kappa$. For $\mathsf A>0$ and small enough $\varepsilon$,~\eqref{eq:learned-horizon} equals
\begin{equation}\label{eq:learned-horizon-expanded}
T_\varepsilon=\left\lceil
\frac{\mathsf A^3}{\varepsilon^3}(1+\upsilon_\varepsilon)^3
\right\rceil,
\end{equation}
where
\begin{align*}
\upsilon_\varepsilon={}&\frac{\mathsf U}{\mathsf A^{3/2}}\varepsilon^{1/2}
+\frac{\mathsf G_q}{\mathsf A^{7/4}}\varepsilon^{3/4}
+\frac{\mathsf D}{\mathsf A^2}\varepsilon
+\frac{\mathsf I}{\mathsf A^3}\varepsilon^2.
\end{align*}
The leading expected gradient budget is $(1+q)\mathsf A_q(\kappa)^3\varepsilon^{-3}$. At the same $q,\ce,\ca$, its ratio to the zero-predictor certificate is
\begin{equation}\label{eq:predictor-budget-ratio}
\left[\frac{\mathsf A_q(\kappa)}{\mathsf A_q(V_C)}\right]^3,
\end{equation}
which is one at $q=1$ and strictly below one for $q<1$ and $\kappa<V_C$. For $\Delta>0$, $q<1$, and $V_C>0$, $\mathsf A_q(V_C)$ attains its minimum over $(\ce,\ca)\in(0,\infty)\times(0,1]$, so the strict inequality also holds between the two coefficients minimized separately over $(\ce,\ca)$. Comparisons across query probabilities include the factors $1+q$.
\label{theory-end}

\section{Experiments}\label{sec:experiments}
We compare randomized transport with full correction on MNIST1D \citep{greydanus2024mnist1d}, using the frozen dataset of $4000$ training and $1000$ test sequences of length $40$ and a GELU network with two hidden layers of width $100$ ($d=15210$). We run Algorithm~1 with the Euclidean LMO, the zero predictor, and schedule~\eqref{eq:alpha-schedule}, for gradient-difference (STORM) and current-point HVP transport. Budgets count gradient-equivalents: a queried gradient difference costs one gradient, and a queried HVP costs $c_H=1.4$ gradients, measured jointly with the fresh gradient. For each method and $q\in\{1,1/2,1/4\}$, $(\eta,\amin)$ is tuned on two seeds over $\eta\in\{0.02,0.05,0.1,0.2\}$ and $\amin\in\{0.003,0.01,0.03,0.1,0.3\}$ and evaluated on ten further seeds shared across $q$, so that ratios to $q=1$ are paired. The floor is capped at $0.3$ because at $\amin=1$ recurrence~\eqref{eq:momentum} discards the correction. Appendix~\ref{app:mnist1d} gives the full protocol.

\textbf{Noise regime.}
To relate the runs to the mechanism behind Proposition~\ref{prop:ratio}, we estimate $\hat\omega_t=\norm{\Cs t}^2/\Et\norm{C_t-\Cs t}^2$ for the gradient-difference correction along a trajectory, the analogue of $L^2/\delta^2$ for the step actually taken. It ranges from $0.02$ to $0.33$ at batch size $4$ and from $0.17$ to $2.6$ at batch size $32$.

\textbf{Results.}
Every $q<1$ reaches a lower final loss than $q=1$ at the same oracle cost, except STORM with $q=1/2$ at batch size $32$, where the difference is not significant (Table~\ref{tab:mnist1d} in Appendix~\ref{app:mnist1d}). At $q=1/4$, the final loss relative to $q=1$ is $0.70$ for STORM and $0.66$ for HVP at batch size $4$, and $0.82$ and $0.57$ at batch size $32$; test accuracy changes by $+1.3$ and $+2.3$ points at batch size $4$, and by $-1.8$ and $+1.3$ points at batch size $32$ (Table~\ref{tab:mnist1d-acc}). For STORM the reduction is larger at batch size $4$, where $\hat\omega$ is smaller, as Proposition~\ref{prop:ratio} suggests; for HVPs, whose queries cost more, it is large at both. At the iteration count of $q=1$ the randomized runs have equal or higher loss, which is consistent with the extra iterations being the source of the gain. The diagonal predictor shows no statistically resolved change in test accuracy and removes at most $4.6\%$ of the residual, in line with~\eqref{eq:irreducible}: the reducible mean part is a fraction $\hat\omega_t/(1+\hat\omega_t)$ of it.

\textbf{Controlled curvature noise.}
To examine the mechanism behind Proposition~\ref{prop:ratio} and the predictor of Section~\ref{sec:predictors}, we use the quadratic $f(x)=\frac12(x-x^*)^\top A(x-x^*)$ in $d=100$, with $A$ diagonal and eigenvalues log-spaced in $[0.01,1]$, so that $L=1$, and the vector oracle $g(x;\xi)=(A+E_\xi)(x-x^*)+\zeta_\xi$. Here $E_\xi=\delta\sqrt d\,uv^\top$ with independent $u,v\sim\mathcal N(0,I/d)$, so that $\E\norm{E_\xi s}^2=\delta^2\norm{s}^2$, and $\E\norm{\zeta_\xi}^2=1$. With a fixed horizon and expected work of about $4000$ gradients (Appendix~\ref{app:synth}), the metric is the run average of $\norm{\nabla f(x_t)}$, bounded in Theorem~\ref{thm:lmo}. The benefit of $q<1$ changes sign as $\delta$ decreases (Figure~\ref{fig:synth}, left): at $\delta=3$, $q=1/4$ lowers the metric by $17\%$, and at $\delta=0.03$ it raises it by $37\%$. The sign changes where $\hat\omega$, measured along the $q=1$ runs, lies between $1.15$ and $1.7$. This is qualitatively consistent with the mechanism behind Proposition~\ref{prop:ratio}, which concerns unnormalized STORM rather than these normalized runs. The diagonal predictor removes most of the mean part of the residual (Figure~\ref{fig:synth}, right). Its residual ratio stays close to the exact-predictor benchmark $1/(1+\hat\omega)$, from $0.99$ down to $0.005$, and at large $\hat\omega$ it restores the performance of $q=1$ (ratio $0.98$ at $\delta=0.03$, $q=1/2$). With unnormalized STORM, a variance-bounded oracle, and the tuning of Theorem~\ref{thm:unnormalized}, the change has the sign of $\cQ(q;\omega)-1$ in $9$ of $10$ settings and, for $\omega<1$, lies within about $2\%$ of the fixed-budget benchmark $\cQ^{2/3}$ (Appendix~\ref{app:prop5}; symmetric noise in Appendix~\ref{app:symmetric}).

\textbf{CIFAR-10.}
We train a ResNet-18 \citep{he2016resnet} with GELU activations \citep{hendrycks2016gelu} and group normalization \citep{wu2018groupnorm} on CIFAR-10 \citep{krizhevsky2009cifar}, with batch size $128$, the same method, and an equal budget of $35100$ gradient-equivalents; a queried HVP costs $c_H=2.32$ gradients on an A100 GPU. Settings are tuned per method and $q$, and test accuracy is evaluated on five seeds (Appendix~\ref{app:cifar}). Relative to $q=1$, test accuracy changes by $+0.96$ points ($95\%$ interval $[0.64,1.28]$) at $q=1/2$ and $+1.52$ points ($[0.96,2.07]$) at $q=1/4$ for HVP transport, whose queries cost the most. For STORM the change is $+0.72$ points ($[-0.01,1.44]$) at $q=1/2$ and $-5.69$ points ($[-7.12,-4.26]$) at $q=1/4$, where the selected step size is nine times smaller than at $q=1$. The diagonal predictor shows no statistically resolved change in test accuracy and removes less than $1\%$ of the importance-weighted residual (Figure~\ref{fig:cifar}).

\begin{figure}[t]
\centering
\includegraphics[width=0.75\columnwidth]{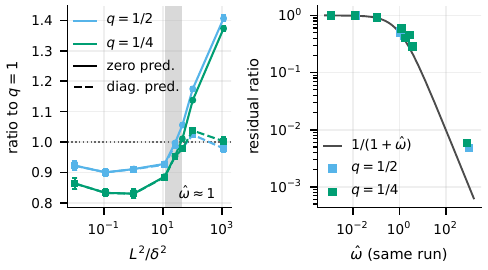}
\caption{Controlled quadratic. Left: run average of $\norm{\nabla f(x_t)}$ relative to $q=1$, zero (solid) and diagonal (dashed) predictor, with paired $95\%$ intervals over $10$ seeds; the band marks $\hat\omega\in[0.9,1.7]$ on the $q=1$ runs. Right: residual ratio of the diagonal predictor (Appendix~\ref{app:mnist1d}) and the conditional benchmark $1/(1+\hat\omega)$ of an exact predictor.}
\label{fig:synth}
\vspace{4pt}
\centering
\includegraphics[width=0.75\columnwidth]{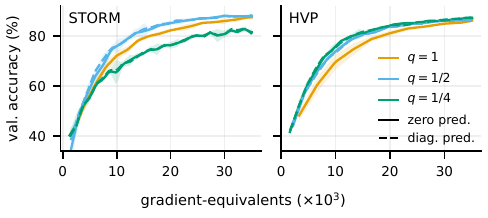}
\caption{CIFAR-10 with ResNet-18: validation accuracy, zero (solid; mean and standard deviation over $5$ seeds) and diagonal (dashed) predictor.}
\label{fig:cifar}
\end{figure}

\section{Conclusion}
Momentum methods need not query their transport correction at every iteration. Querying it with probability $q$ and otherwise using a cheap prediction preserves its conditional mean, costs $(q^{-1}-1)$ times the prediction loss in variance, and keeps the $O(T^{-1/3})$ rate for fixed $q$. The saved queries buy iterations, which helps most with expensive or noisy corrections.
\label{main-end}

\clearpage

\subsection*{AI Use Statement}
We used a large language model as an assistant for restructuring and editing the manuscript, checking
and tightening proofs, drafting parts of the experiment code, and suggesting related work. The authors
verified all such output. The research ideas, the method and the main results are the authors' own,
and the authors take full responsibility for the content of this work.

\bibliography{references}
\clearpage
\appendix
\onecolumn
\section*{Supplementary Material}

\section{Transport Identities and Special Cases}\label{app:transport}
All conditional expectations use the history of Section~\ref{sec:framework}; in particular $s_t$ and $m_{t-1}$ are $\F_{t-1}$-measurable. Let $\chi_t=J_t/q-1$. By Assumption~\ref{ass:query},
\[
\E[\chi_t\mid\F_{t-1},C_t,B_t,g_t]=0,\qquad
\E[\chi_t^2\mid\F_{t-1},C_t,B_t,g_t]=q^{-1}-1.
\]
Thus $\widehat C_t=C_t+\chi_t(C_t-B_t)$ has conditional mean $\mu_t$. For any $\F_{t-1}$-measurable center $z_t$,
\begin{align*}
\Et\norm{\widehat C_t-z_t}^2
&=\Et\norm{C_t-z_t}^2
+2\Et[\chi_t\ip{C_t-z_t}{C_t-B_t}]+\Et[\chi_t^2\norm{C_t-B_t}^2]\\
&=\Et\norm{C_t-z_t}^2+(q^{-1}-1)\Et\norm{C_t-B_t}^2.
\end{align*}
The same conditioning gives
\[
\Et\ip{\zeta_t}{\widehat C_t-\mu_t}
=\Et\ip{\zeta_t}{C_t-\mu_t}.
\]
These calculations prove Lemma~\ref{lem:identities}, including predictors that use cheap current-sample information. For predictable $B_t$, expanding $C_t-B_t=(C_t-\mu_t)+(\mu_t-B_t)$ proves~\eqref{eq:irreducible}.

\subsection{Gradient-Difference Transport}
Unbiasedness at both points gives $\mu_t=\Cs t$, and~\eqref{eq:similarity} gives the centered variance bound with $v=\delta$. $L$-smoothness gives $\norm{\Cs t}\le L\norm{s_t}$, hence
\[
\Et\norm{C_t}^2=\Et\norm{C_t-\Cs t}^2+\norm{\Cs t}^2
\le(\delta^2+L^2)\norm{s_t}^2.
\]
For a zero predictor,
\begin{align}
\Et\norm{J_tC_t/q-\Cs t}^2
&=\frac1q\Et\norm{C_t-\Cs t}^2
+\frac{1-q}{q}\norm{\Cs t}^2
\le V_q^2\norm{s_t}^2.\label{eq:app-storm-var}
\end{align}
By Cauchy--Schwarz and Lemma~\ref{lem:identities},
\begin{equation}\label{eq:app-storm-cov}
\left|\Et\ip{\zeta_t}{J_tC_t/q-\Cs t}\right|
=\left|\Et\ip{\zeta_t}{C_t-\Cs t}\right|
\le\sigma\delta\norm{s_t}.
\end{equation}

\subsection{Second-Order Transport at Either Endpoint}
For $s=x-y$, the fundamental theorem of calculus gives
$\nabla f(x)-\nabla f(y)=\int_0^1\nabla^2f(y+\vartheta s)s\,d\vartheta$.
For an oracle at $y$, the mean correction is $\nabla^2f(y)s$ and
\[
\norm{\nabla^2f(y)s-[\nabla f(x)-\nabla f(y)]}
\le\int_0^1\rho\vartheta\norm{s}^2\,d\vartheta=\frac\rho2\norm{s}^2.
\]
For an oracle at $x$, replace $\vartheta$ by $1-\vartheta$. Both endpoints therefore satisfy~\eqref{eq:bias-assumption} with $b=\rho/2$, and the variance bound at the chosen endpoint gives $v=\tau$. Since a twice differentiable $L$-smooth function has $\norm{\nabla^2f(x)}_{\mathrm{op}}\le L$, both endpoints satisfy $\norm{\mu_t}\le L\norm{s_t}$. For a zero predictor, the centered randomized variance is at most $[\tau^2+(1-q)L^2]q^{-1}\norm{s_t}^2$.

\section{Tracking and LMO Convergence}\label{app:lmo}
We first prove a tracking bound for a general deterministic sequence $\alpha_t\in(0,1]$ with $\alpha_0=1$. Set $\beta_t=1-\alpha_t$ and, for $0\le i\le t$,
\begin{equation}\label{eq:weights}
p_{i,t}=\prod_{j=i+1}^t\beta_j,
\qquad p_{t,t}=1.
\end{equation}

\subsection{General Deterministic Momentum Schedule}
\begin{lemma}[Weighted tracking bound]\label{lem:tracking}
Suppose
\[
e_t=\beta_t e_{t-1}+\beta_t r_t+w_t,
\qquad \Et w_t=0,
\qquad \E\norm{e_0}^2\le\sigma^2,
\]
where $r_t$ is $\F_{t-1}$-measurable and $w_t$ is $\F_t$-measurable. Let $\bar r_i=\E\norm{r_i}$, and let $\bar w_i\ge\E\norm{w_i}^2$ be deterministic. Then
\begin{equation}\label{eq:general-tracking}
\frac1T\sum_{t=0}^{T-1}\E\norm{e_t}
\le\frac1T\sum_{t=0}^{T-1}\left[
\sigma p_{0,t}+\sum_{i=1}^t p_{i,t}\beta_i\bar r_i
+\sqrt{\sum_{i=1}^t p_{i,t}^2\bar w_i}\right].
\end{equation}
If $e_t=m_t-\nabla f(x_t)$, Assumption~\ref{ass:objective} and the LMO update~\eqref{eq:lmo-update} give a stationarity bound obtained by multiplying the right side by $2R$ and adding $\Delta/(\eta T)+L\eta R^2/2$.
\end{lemma}
\begin{proof}
Unrolling the recursion gives
\[
e_t=p_{0,t}e_0+\mathcal D_t+\mathcal M_t,
\qquad \mathcal D_t=\sum_{i=1}^t p_{i,t}\beta_i r_i,
\qquad \mathcal M_t=\sum_{i=1}^t p_{i,t}w_i.
\]
All weights are deterministic. For $i<j$, $w_i$ is $\F_{j-1}$-measurable, so the tower property gives $\E\ip{w_i}{w_j}=0$ and
\[
\E\norm{\mathcal M_t}^2=\sum_{i=1}^t p_{i,t}^2\E\norm{w_i}^2
\le\sum_{i=1}^t p_{i,t}^2\bar w_i.
\]
The triangle inequality and Cauchy--Schwarz give~\eqref{eq:general-tracking}; no independence between $e_0$, $\mathcal D_t$, and $\mathcal M_t$ is needed.

For the exact LMO direction $u_t$ and a maximizer $u_t^*$ of $\ip{\nabla f(x_t)}{u}$ over $\K$,
\[
\ip{\nabla f(x_t)}{u_t}
=\ip{m_t}{u_t}-\ip{e_t}{u_t}
\ge\ip{m_t}{u_t^*}-\ip{e_t}{u_t}
\ge\sk(\nabla f(x_t))-2R\norm{e_t}.
\]
Smoothness and $\norm{u_t}\le R$ yield
\[
f(x_{t+1})\le f(x_t)-\eta\sk(\nabla f(x_t))
+2\eta R\norm{e_t}+L\eta^2R^2/2.
\]
Summing over $t=0,\ldots,T-1$, taking expectations, and using $f(x_T)\ge f_*$ proves the last assertion.
\end{proof}

For the transport recurrence, $r_i=\mu_i-\Cs i$ and
$w_i=\beta_i(\widehat C_i-\mu_i)+\alpha_i\zeta_i$. Let $\ell_i=\E\norm{C_i-B_i}^2$. Lemma~\ref{lem:identities} and Assumptions~\ref{ass:objective}--\ref{ass:transport} give the valid choices
\begin{align}
\bar r_i&\le b\eta^2R^2,\notag\\
\bar w_i&=\alpha_i^2\sigma^2+2\alpha_i\beta_i\sigma v\eta R
+\beta_i^2\left[v^2\eta^2R^2+(q^{-1}-1)\ell_i\right].\label{eq:noise-bound}
\end{align}
Together with Lemma~\ref{lem:tracking}, these give the explicit bound for any deterministic schedule. Determinism of the schedule ensures that the weights preserve martingale orthogonality.

\subsection{Bounds for the Decreasing Schedule}
For~\eqref{eq:alpha-schedule}, let $\bmax=1-\amin$. For every $t\ge0$,
\begin{align}
p_{0,t}&\le\frac{2}{(t+1)(t+2)},
&\sum_{t=0}^{T-1}p_{0,t}&\le2,\label{eq:initial-weights}\\
\sum_{i=1}^t p_{i,t}^2\alpha_i^2&\le\alpha_t,
&2\sum_{i=1}^t p_{i,t}^2\alpha_i\beta_i&\le1,\label{eq:noise-weights}
\end{align}
and, for $1\le i\le T-1$,
\begin{equation}\label{eq:tail-weights}
\sum_{t=i}^{T-1}p_{i,t}\beta_i\le\frac{\bmax}{\amin},
\qquad
\sum_{t=i}^{T-1}p_{i,t}^2\beta_i^2\le\frac{\bmax^2}{\amin(2-\amin)}.
\end{equation}
Also,
\begin{equation}\label{eq:pointwise-correction-weights}
\sum_{i=1}^t p_{i,t}^2\beta_i^2
\le\frac{\bmax^2}{\amin(2-\amin)}.
\end{equation}

For~\eqref{eq:initial-weights}, use $\beta_j\le j/(j+2)$ and telescope the product and then the sum. For the first inequality in~\eqref{eq:noise-weights}, put $\Upsilon_0=0$ and $\Upsilon_t=\beta_t^2\Upsilon_{t-1}+\alpha_t^2$, so that $\Upsilon_t=\sum_{i=1}^tp_{i,t}^2\alpha_i^2$. Since $1/\alpha_t=\min\{1/\amin,(t+2)/2\}$,
\[
0\le\frac1{\alpha_t}-\frac1{\alpha_{t-1}}\le\frac12,
\quad\text{so}\quad
\alpha_{t-1}\le\frac{\alpha_t}{1-\alpha_t/2}.
\]
Induction and $\beta_t\le1-\alpha_t/2$ give
\[
\Upsilon_t\le\frac{\beta_t^2\alpha_t}{1-\alpha_t/2}+\alpha_t^2
\le\alpha_t\beta_t+\alpha_t^2=\alpha_t.
\]
For the second inequality, $2\alpha_i\beta_i\le1-\beta_i^2$ and $p_{i-1,t}=\beta_ip_{i,t}$ give
\[
\sum_{i=1}^t p_{i,t}^2(1-\beta_i^2)
=\sum_{i=1}^t(p_{i,t}^2-p_{i-1,t}^2)=1-p_{0,t}^2\le1.
\]
Finally, $p_{i,t}\le\bmax^{t-i}$ and $\beta_i\le\bmax$ prove~\eqref{eq:tail-weights} and~\eqref{eq:pointwise-correction-weights} through geometric series. These arguments include $\amin=1$, where every $\beta_i=0$.

\subsection{Proof of Theorem~\ref{thm:lmo}}
By~\eqref{eq:initial-weights}, the initial error contributes at most $2\sigma/T$ to the average tracking norm. By~\eqref{eq:tail-weights}, after exchanging the finite sums, the average drift norm is at most $\bmax b\eta^2R^2/\amin$.

For the martingale sum, define
\[
\mathcal V_t=\sigma^2\amin+\sigma v\eta R
+v^2\eta^2R^2\sum_{i=1}^t p_{i,t}^2\beta_i^2
+(q^{-1}-1)\sum_{i=1}^t p_{i,t}^2\beta_i^2\ell_i.
\]
Equations~\eqref{eq:noise-bound}--\eqref{eq:noise-weights} and
$\alpha_t\le\amin+2/(t+2)$ imply, for $t\ge1$,
\[
\E\norm{\mathcal M_t}^2\le\mathcal V_t+\frac{2\sigma^2}{t+2}.
\]
Using~\eqref{eq:tail-weights} and~\eqref{eq:pointwise-correction-weights}, exchanging the finite sums, and using $T^{-1}\sum_i\ell_i=\cL_T$,
\[
\frac1T\sum_{t=1}^{T-1}\mathcal V_t\le W_T.
\]
Cauchy--Schwarz, $\mathcal M_0=0$, and $\sqrt{x+y}\le\sqrt x+\sqrt y$ give
\begin{align*}
\frac1T\sum_{t=0}^{T-1}\E\norm{\mathcal M_t}
&\le\sqrt{\frac1T\sum_{t=1}^{T-1}\mathcal V_t}
+\frac{\sqrt2\sigma}{T}\sum_{t=1}^{T-1}\frac1{\sqrt{t+2}}
\le\sqrt{W_T}+\frac{2\sqrt2\sigma}{\sqrt T},
\end{align*}
where $\sum_{t=1}^{T-1}(t+2)^{-1/2}\le2\sqrt T$ by comparison with $\int x^{-1/2}dx$. Combining the three contributions,
\begin{equation}\label{eq:average-error}
\frac1T\sum_{t=0}^{T-1}\E\norm{e_t}
\le\frac{2\sigma}{T}+\frac{\bmax b\eta^2R^2}{\amin}
+\sqrt{W_T}+\frac{2\sqrt2\sigma}{\sqrt T}.
\end{equation}
The descent statement of Lemma~\ref{lem:tracking} gives~\eqref{eq:lmo-bound}. Only $m_0,\ldots,m_{T-1}$ are used, and for $T=1$ all correction sums are empty.

\subsection{Inexact Directions}\label{app:inexact}
\paragraph{Objective gap.}
Suppose $u_t\in\K$ and $\ip{m_t}{u_t}\ge\sk(m_t)-\epsilon_t$. With $u_t^*$ as in the proof of Lemma~\ref{lem:tracking},
\[
\ip{\nabla f(x_t)}{u_t}
=\ip{m_t}{u_t}-\ip{e_t}{u_t}
\ge\ip{m_t}{u_t^*}-\epsilon_t-\ip{e_t}{u_t}
\ge\sk(\nabla f(x_t))-2R\norm{e_t}-\epsilon_t.
\]
Since $\norm{u_t}\le R$, all other steps are unchanged, and the descent inequality gains $\eta\epsilon_t$. Dividing by $\eta T$ adds $T^{-1}\sum_t\E\epsilon_t$ to~\eqref{eq:lmo-bound}.

\paragraph{Additive error.}
Suppose $\norm{u_t-u_t^\circ}\le\epsilon_{\rm L}<\rK$, where $u_t^\circ$ is an exact LMO direction for $m_t$. The exact alignment inequality applies to $u_t^\circ$, and $\sk(z)\ge\rK\norm{z}$ gives
\[
\ip{\nabla f(x_t)}{u_t}
\ge\ip{\nabla f(x_t)}{u_t^\circ}-\epsilon_{\rm L}\norm{\nabla f(x_t)}
\ge\left(1-\frac{\epsilon_{\rm L}}{\rK}\right)\sk(\nabla f(x_t))-2R\norm{e_t}.
\]
Moreover $\norm{u_t}\le R+\epsilon_{\rm L}$, so $\norm{s_t}\le\eta(R+\epsilon_{\rm L})$, and smoothness gives
\[
f(x_{t+1})\le f(x_t)-\eta\left(1-\frac{\epsilon_{\rm L}}{\rK}\right)\sk(\nabla f(x_t))
+2\eta R\norm{e_t}+\frac{L\eta^2(R+\epsilon_{\rm L})^2}{2}.
\]
The tracking analysis uses the direction set only through $\norm{s_t}\le\eta R$, which now holds with $R+\epsilon_{\rm L}$. Every term on the right side of~\eqref{eq:lmo-bound} is nondecreasing in $R$, so replacing $R$ by $R+\epsilon_{\rm L}$ throughout bounds $(1-\epsilon_{\rm L}/\rK)\E\sk(\nabla f(x_U))$.

\subsection{Explicit Rate}
Under Corollary~\ref{cor:rate}, $\amin(2-\amin)\ge\amin$ and $\bmax\le1$ give
\begin{align*}
\E\sk(\nabla f(x_U))\le{}&
\left[\frac\Delta\ce+2R\sqrt{\ca\sigma^2+\sigma v\ce R+
\frac{\ce^2R^2}{\ca}\bigl(v^2+(q^{-1}-1)\bar\kappa^2\bigr)}\right]T^{-1/3}\\
&+4\sqrt2R\sigma T^{-1/2}
+\left[\frac{L\ce R^2}{2}+\frac{2b\ce^2R^3}{\ca}\right]T^{-2/3}
+4R\sigma T^{-1},
\end{align*}
which is~\eqref{eq:rate}.

\section{Proof of the Unnormalized STORM Bound}\label{app:unnormalized}
This appendix uses $x_{t+1}=x_t-\eta m_t$, $B_t=0$, and schedule~\eqref{eq:alpha-schedule} with $\amin=c\eta^2$. Finite second moments of $x_t$ and $m_t$ follow by induction from the oracle bounds, smoothness, and $\norm{s_t}=\eta\norm{m_{t-1}}$. Write $\nabla_t=\nabla f(x_t)$ and $e_t=m_t-\nabla_t$. Since gradient-difference transport is unbiased,~\eqref{eq:app-storm-var}--\eqref{eq:app-storm-cov} give, for one transition,
\[
\E[\norm{e_{t+1}}^2\mid\F_t]
\le \beta_{t+1}^2\norm{e_t}^2+\alpha_{t+1}^2\sigma^2
+2\alpha_{t+1}\beta_{t+1}\sigma\delta\eta\norm{m_t}
+\beta_{t+1}^2V_q^2\eta^2\norm{m_t}^2.
\]
Smoothness and $2\ip{\nabla_t}{m_t}=\norm{\nabla_t}^2+\norm{m_t}^2-\norm{e_t}^2$ give
\begin{equation}\label{eq:euclidean-descent}
f(x_{t+1})\le f(x_t)-\frac\eta2\norm{\nabla_t}^2
-\frac{\eta(1-L\eta)}2\norm{m_t}^2+\frac\eta2\norm{e_t}^2.
\end{equation}

Define $\Phi_t=\E[f(x_t)-f_*]+(\eta/\alpha_t)\E\norm{e_t}^2$. For $t\le T-2$, the coefficient of $\eta\E\norm{e_t}^2$ in $\Phi_{t+1}-\Phi_t$ is
\begin{equation}\label{eq:potential-coefficient}
\frac12+\frac{(1-\alpha_{t+1})^2}{\alpha_{t+1}}-\frac1{\alpha_t}
=\left(\frac1{\alpha_{t+1}}-\frac1{\alpha_t}\right)-\frac32+\alpha_{t+1}.
\end{equation}
Since $c\eta^2\le1$, $1/\alpha_t=\min\{1/(c\eta^2),(t+2)/2\}$ for every $t\ge0$, including $\alpha_0=1$. The reciprocal increments therefore lie in $[0,1/2]$, and since $\alpha_{t+1}\le1$,~\eqref{eq:potential-coefficient} is nonpositive.

Since $\alpha_{t+1}\ge c\eta^2$ and $\beta_{t+1}^2\le1$, the quadratic term in $m_t$ is at most $-\eta\theta_\eta\E\norm{m_t}^2$, and the covariance term is at most $2\sigma\delta\eta^2\E\norm{m_t}$. Completing the square pointwise,
\[
-\eta\theta_\eta\norm{m_t}^2+2\sigma\delta\eta^2\norm{m_t}
\le\frac{\sigma^2\delta^2\eta^3}{\theta_\eta}.
\]
Hence
\begin{equation}\label{eq:potential-step}
\Phi_{t+1}\le\Phi_t-\frac\eta2\E\norm{\nabla_t}^2
+\eta\alpha_{t+1}\sigma^2
+\frac{\sigma^2\delta^2\eta^3}{\theta_\eta},\qquad 0\le t\le T-2.
\end{equation}

At the final step no $m_T$ is formed. Apply~\eqref{eq:euclidean-descent} and compare $\E[f(x_T)-f_*]$ with $\Phi_{T-1}$. The error coefficient is $\eta(1/2-1/\alpha_{T-1})\le0$, and $\theta_\eta>0$ implies $1-L\eta>0$. Hence
\[
\E[f(x_T)-f_*]\le\Phi_{T-1}-\frac\eta2\E\norm{\nabla_{T-1}}^2,
\]
which also covers $T=1$. Telescoping and $\Phi_0\le\Delta+\eta\sigma^2$ give
\[
\frac1T\sum_{t=0}^{T-1}\E\norm{\nabla_t}^2
\le\frac{2\Delta}{\eta T}+\frac{2\sigma^2}{T}
+\frac{2\sigma^2}{T}\sum_{t=1}^{T-1}\alpha_t
+\frac{2(T-1)\sigma^2\delta^2\eta^2}{T\theta_\eta}.
\]
Finally,
\[
\sum_{t=1}^{T-1}\alpha_t
\le c\eta^2T+2\sum_{t=1}^{T-1}\frac1{t+2}
\le c\eta^2T+2\log\frac{T+2}{2},
\]
and substitution proves Theorem~\ref{thm:unnormalized}.

\subsection{Leading Constants and Feasible Parameters}
Let $V=V_q>0$ and $\delta>0$. For $c>2V^2$, the coefficient of $\eta^2$ in~\eqref{eq:unnormalized} at $L=0$ is
\[
\Xi_q(c)=2\sigma^2\left[c+\frac{\delta^2}{1/2-V^2/c}\right].
\]
Writing $c=2V^2+y$ with $y>0$ gives
\[
\frac{\Xi_q(c)}{2\sigma^2}=2V^2+y+2\delta^2+\frac{4\delta^2V^2}{y},
\]
minimized at $y=2\delta V$. This proves~\eqref{eq:optimal-c} with $K_q=\Xi_q(c_q)$, and
\[
\theta_0=\frac12-\frac{V^2}{c_q}=\frac{\delta}{2(V+\delta)}>0,
\qquad K_q\theta_0^2=\sigma^2\delta^2.
\]
The cap~\eqref{eq:cap-threshold} ensures $c_q\eta^2\le1$ and $L\eta\le\theta_0$, hence $\theta_\eta=\theta_0-L\eta/2\ge\theta_0/2$. Set
\[
\eta_T^\star=(\Delta/(K_qT))^{1/3},\qquad
\eta=\min\{\eta_{\max},\eta_T^\star\}.
\]
The excess over the limiting noise coefficient is
\begin{align*}
2\sigma^2\eta^2\left(c_q+\frac{\delta^2}{\theta_\eta}\right)-K_q\eta^2
&=\frac{\sigma^2\delta^2L\eta^3}{\theta_0\theta_\eta}
\le\frac{2\sigma^2\delta^2L\eta^3}{\theta_0^2}
=2LK_q\eta^3\le\frac{2L\Delta}{T}.
\end{align*}
For $T\ge T_{\rm cap}$, $\eta=\eta_T^\star$ and
$2\Delta/(\eta_T^\star T)+K_q(\eta_T^\star)^2=3(\Delta^2K_q)^{1/3}T^{-2/3}$.
Substitution in~\eqref{eq:unnormalized} proves~\eqref{eq:leading-rate}. For every $T\ge1$, the capped step satisfies
\begin{align}
\E\norm{\nabla f(x_U)}^2\le{}&3(\Delta^2K_q)^{1/3}T^{-2/3}
+\frac{2\Delta}{T}\left(\frac1{\eta_{\max}}-\frac1{\eta_T^\star}\right)_+
+\frac{2L\Delta+2\sigma^2+4\sigma^2\log((T+2)/2)}{T},\label{eq:storm-all-T}
\end{align}
where $(z)_+=\max\{z,0\}$, since
$1/\eta=1/\eta_T^\star+(1/\eta_{\max}-1/\eta_T^\star)_+$ and $\eta\le\eta_T^\star$.

\paragraph{Finite accuracy budget.}
Put
\[
\varphi(t)=\frac{2L\Delta+2\sigma^2+4\sigma^2\log((t+2)/2)}{t}.
\]
For $t\ge40$, $\log((t+2)/2)\ge\log21>3$, and
\[
\frac{d}{dt}[t^{2/3}\varphi(t)]
=t^{-4/3}\left[\frac{4\sigma^2t}{t+2}
-\frac{2L\Delta+2\sigma^2+4\sigma^2\log((t+2)/2)}3\right]\le0.
\]
Since $T_0^{\rm S}\ge3\sqrt3\Delta\sqrt{K_q}\varepsilon^{-3}$ gives $3(\Delta^2K_q)^{1/3}(T_0^{\rm S})^{-2/3}\le\varepsilon^2$, for $t\ge T_0^{\rm S}$ the right side of~\eqref{eq:leading-rate} is at most $(T_0^{\rm S}/t)^{2/3}(\varepsilon^2+\varphi_0)$. The definition of $T_\varepsilon^{\rm S}$ makes this at most $\varepsilon^2$, and Jensen's inequality gives $\E\norm{\nabla f(x_U)}\le\varepsilon$. The count~\eqref{eq:count} gives the expected budget.

For small $\varepsilon$, let $\Lambda_q=3\sqrt3\Delta\sqrt{K_q}$, so that $T_0^{\rm S}=\Lambda_q\varepsilon^{-3}$. The budget is then exactly
\[
(1+q)\left\lceil \Lambda_q\varepsilon^{-3}
\left[1+\frac{\varepsilon}{\Lambda_q}
\left(2L\Delta+2\sigma^2+4\sigma^2\log\frac{\Lambda_q\varepsilon^{-3}+2}{2}\right)\right]^{3/2}\right\rceil-q.
\]
The bracket tends to one, which yields the leading term~\eqref{eq:leading-budget}. Theorem~\ref{thm:unnormalized} itself also holds when $\Delta$, $\sigma$, or $\delta$ vanishes; the optimized formulas assume $\Delta,\sigma,\delta>0$.

\subsection{Proof of Proposition~\ref{prop:ratio}}
At $q=1$, $V_1=\delta$, and dividing the leading coefficients gives~\eqref{eq:ratio}. Differentiation at $q=1$ gives
\[
\left.\frac{\partial\cQ(q;\omega)}{\partial q}\right|_{q=1}=\frac{1-\omega}{4}.
\]
If $\omega<1$, decreasing $q$ slightly from one strictly lowers the ratio. Conversely, for every $q<1$, $\cQ(q;\omega)$ increases with $\omega$, so it suffices to show $\cQ(q;1)\ge1$, that is,
\[
\sqrt{2/q-1}\ge\frac{3-q}{1+q}.
\]
Both sides are nonnegative. Squaring and clearing positive denominators reduces the claim to
\[
(2-q)(1+q)^2-q(3-q)^2=2(1-q)^3\ge0.
\]
In the limit $\omega\to0$, the interior minimizer of $\cQ(q;0)$ satisfies $2q^{3/2}+q-1=0$, so $q\approx0.4320408$ with ratio $0.9026797$.

\section{Proof of the Online-Predictor Guarantee}\label{app:regret}
Put $V_C^2=v^2+M^2$. Since $\norm{\mu_t}\le M\norm{s_t}$ and $\Et\norm{C_t-\mu_t}^2\le v^2\norm{s_t}^2$, we have $\Et\norm{y_t}^2\le V_C^2$. If $s_t=0$, the mean and variance of $C_t$ vanish, so the zero convention is consistent. For $s_t\ne0$, $\norm{z_t}=1$, hence $\norm{h_t\odot z_t}\le\hmax$ and $\norm{z_t}_\infty\le1$. Assumption~\ref{ass:query} gives
\begin{align*}
\E\norm{\widehat a_t}^2
=q^{-1}\E\norm{\nabla\ell_t(h_t)}^2
=4q^{-1}\E\norm{(h_t\odot z_t-y_t)\odot z_t}^2
\le8q^{-1}(\hmax^2+V_C^2).
\end{align*}

For any fixed $h\in[-\hmax,\hmax]^d$, nonexpansiveness of the Euclidean projection gives
\[
\norm{h_{t+1}-h}^2
\le\norm{h_t-h}^2-2\eta_h\ip{\widehat a_t}{h_t-h}
+\eta_h^2\norm{\widehat a_t}^2.
\]
Since $h_t$ is predictable and $h$ is fixed, Assumption~\ref{ass:query} and convexity give
\[
\E\ip{\widehat a_t}{h_t-h}
=\E\ip{\nabla\ell_t(h_t)}{h_t-h}
\ge\E[\ell_t(h_t)-\ell_t(h)].
\]
Summing over $t=1,\ldots,n$ and using $h_1=0$ and $\norm{h}^2\le d\hmax^2$ proves~\eqref{eq:regret}; minimizing over $\eta_h>0$ gives~\eqref{eq:regret-value}. The argument holds along the adaptive sequence of iterates and does not require an oblivious loss sequence.

\subsection{Explicit Prediction Loss}
For the transport residual,
\[
\norm{C_t-B_t}^2=\norm{s_t}^2\ell_t(h_t)
\le\eta^2R^2\ell_t(h_t).
\]
Summing, applying~\eqref{eq:regret} for each fixed comparator, and taking the infimum outside the expectations proves~\eqref{eq:online-conversion}. With~\eqref{eq:comparator-loss}--\eqref{eq:Gamma} this is the first line of~\eqref{eq:explicit-LT}; the second follows from $(T-1)/T\le1$ and $\sqrt{T-1}/T\le T^{-1/2}$. For $h=0$, $\E\ell_t(0)=\E\norm{y_t}^2\le V_C^2$, so $\kappa_T^2\le V_C^2$.

\subsection{When Learning Pays Off}
For a zero predictor, the variance decomposition and the mean bound give
$\E\norm{C_t}^2\le V_C^2\E\norm{s_t}^2\le V_C^2\eta^2R^2$.
Summing over the $n=T-1$ correction indices gives the certificate $\eta^2R^2(n/T)V_C^2$. For the learner,~\eqref{eq:explicit-LT} and $\kappa_T\le\kappa$ give~\eqref{eq:prediction-threshold}. If~\eqref{eq:break-even} holds, then $\Gamma_q/\sqrt n<V_C^2-\kappa^2$, so the residual certificate improves strictly. Its coefficient in~\eqref{eq:WT} is positive when $q<1$ and $\amin<1$, so the stationarity certificate also improves. Substituting~\eqref{eq:Gamma} gives~\eqref{eq:query-threshold}. Both certificates use the same parameters and the same constants $\Delta,L,\sigma,v,M$, and each bounds its own trajectory.

\subsection{Proof of Corollary~\ref{cor:learned-rate}}
Substitute~\eqref{eq:explicit-LT} in~\eqref{eq:WT}. With $\bmax\le1$ and $\amin(2-\amin)\ge\amin$,
\begin{equation}\label{eq:learned-WT}
W_T\le T^{-2/3}\left[Z_q(\kappa_T)+\Psi_qT^{-1/2}\right],
\qquad \Psi_q=\frac{\ce^2R^2}{\ca}(q^{-1}-1)\Gamma_q.
\end{equation}
The square-root inequality gives
\[
2R\sqrt{W_T}\le2R\sqrt{Z_q(\kappa_T)}T^{-1/3}
+2R\sqrt{\Psi_q}\,T^{-7/12},
\]
and $2R\sqrt{\Psi_q}=\mathsf G_q$. Substituting the remaining terms of Theorem~\ref{thm:lmo} proves~\eqref{eq:learned-rate}. When $Z_q(\kappa_T)>0$,
\[
\sqrt{Z_q(\kappa_T)+\Psi_qT^{-1/2}}-\sqrt{Z_q(\kappa_T)}
\le\frac{\Psi_qT^{-1/2}}{2\sqrt{Z_q(\kappa_T)}},
\]
and multiplying by $2RT^{-1/3}$ gives~\eqref{eq:sharp-regret}.

\subsection{Finite Oracle Complexity}
Assume $\kappa_T\le\kappa$ for every horizon considered. Monotonicity of $Z_q$ gives
\[
\E\sk(\nabla f(x_U))\le \mathsf AT^{-1/3}+\mathcal E(T),
\qquad \mathsf A=\mathsf A_q(\kappa).
\]
For $t\ge T_0$, each nonnegative term of $\mathcal E(t)$ is at most $(T_0/t)^{1/3}$ times its value at $T_0$, and $\mathsf AT_0^{-1/3}\le\varepsilon$, so
\[
\mathsf At^{-1/3}+\mathcal E(t)
\le(T_0/t)^{1/3}[\varepsilon+\mathcal E(T_0)].
\]
The real number inside the ceiling in~\eqref{eq:learned-horizon} makes the right side equal to $\varepsilon$, and rounding up preserves the inequality. The parameters $\eta,\amin,\eta_h$ are chosen for this integer horizon; the uniform bound $\kappa$ makes the prescription independent of the resulting trajectory, and $\kappa=V_C$ always provides one.

Equation~\eqref{eq:count} gives the gradient-difference count; the second-order variant uses exactly $T_\varepsilon$ fresh gradients and an expected $q(T_\varepsilon-1)$ HVPs, and the predictor makes no oracle calls. For $\mathsf A>0$ and $\varepsilon$ small enough that $(\mathsf A/\varepsilon)^3\ge2$, substituting $T_0=(\mathsf A/\varepsilon)^3$ into $\mathcal E(T_0)/\varepsilon$ gives the four terms of $\upsilon_\varepsilon$ in~\eqref{eq:learned-horizon-expanded}. Since $\upsilon_\varepsilon\to0$, the leading coefficient is $(1+q)\mathsf A^3$. A zero predictor has $\cL_T\le\eta^2R^2V_C^2$ and no regret term, which gives~\eqref{eq:predictor-budget-ratio} for positive leading coefficients.

For the comparison after optimizing $(\ce,\ca)$: if $\Delta>0$, $q<1$, and $V_C>0$, then $V_C^2(q^{-1}-1)>0$, and $\mathsf A_q(V_C)$ tends to infinity as $\ce\to0$, as $\ce\to\infty$, and as $\ca\to0$. It is continuous on $(0,\infty)\times(0,1]$, so it attains its minimum at some $(\ce^\star,\ca^\star)$. For $\kappa<V_C$, $\inf\mathsf A_q(\kappa)\le\mathsf A_q(\kappa;\ce^\star,\ca^\star)<\mathsf A_q(V_C;\ce^\star,\ca^\star)$. The condition $V_C>0$ is needed: if $v=M=0$, then $\mathsf A_q(V_C)=\Delta/\ce+2R\sigma\sqrt{\ca}$, whose infimum is zero and is not attained; the strict comparison assumes $0\le\kappa<V_C$, which implies $V_C>0$.

If $Z_q(\kappa)>0$, one may use
\[
\widetilde{\mathsf G}_q=\frac{\ce^2R^3(q^{-1}-1)\Gamma_q}{\ca\sqrt{Z_q(\kappa)}}
\]
and replace $\mathsf G_q t^{-7/12}$ in $\mathcal E(t)$ by
$\min\{\mathsf G_qt^{-7/12},\widetilde{\mathsf G}_qt^{-5/6}\}$. Equation~\eqref{eq:learned-WT} holds with $\kappa$ in place of $\kappa_T$, both powers decay faster than $t^{-1/3}$, and the same horizon argument applies. In the small-accuracy expression, the corresponding term of $\upsilon_\varepsilon$ is $\widetilde{\mathsf G}_q\mathsf A^{-5/2}\varepsilon^{3/2}$, or the minimum of this and $\mathsf G_q\mathsf A^{-7/4}\varepsilon^{3/4}$.

\section{MARS Scaling and Residual Scaling}\label{app:scaling}
We combine query randomization with a nonnegative correction strength $\gamma$. The first variant scales the entire correction, as in the first-moment recurrence of MARS \citep{yuan2025mars,shulgin2026mars}; the second scales only the prediction residual. Both use the general sequence $\alpha_t$ of Algorithm~1; their rate bounds use schedule~\eqref{eq:alpha-schedule} with $\beta_t=1-\alpha_t$ and $\bmax=1-\amin$. Throughout this appendix, $\gamma$ is constant within each run, $q\in(0,1]$ is fixed, $m_0=g_0$, and the parameter update is the LMO step~\eqref{eq:lmo-update}.

The loss $\cL_T$ is computed along the trajectory of the variant under consideration. Define
\begin{align}
\Theta_T&=\frac{\Delta}{\eta T}+\frac{L\eta R^2}{2}
+\frac{4R\sigma}{T}+\frac{4\sqrt2R\sigma}{\sqrt T},\label{eq:scaling-common}\\
W_T^\gamma&=\sigma^2\amin+\gamma\sigma v\eta R
+\frac{\gamma^2\bmax^2}{\amin(2-\amin)}
\left[v^2\eta^2R^2+(q^{-1}-1)\cL_T\right].\label{eq:scaling-noise}
\end{align}

\subsection{Scaling the Entire Correction}
The randomized MARS-style recurrence is
\begin{equation}\label{eq:mars-scaled}
m_t=\beta_t\bigl(m_{t-1}+\gamma\widehat C_t\bigr)+\alpha_t g_t.
\end{equation}
At $q=1$ the predictor cancels. For gradient-difference transport this is the MARS first moment before clipping or preconditioning, and $\gamma=1$ and $\gamma=0$ recover randomized transport and ordinary momentum.

\begin{proposition}[LMO bound with standard scaling]\label{prop:mars-scaling}
Under Assumptions~\ref{ass:objective}--\ref{ass:query}, update~\eqref{eq:mars-scaled} with schedule~\eqref{eq:alpha-schedule} and constant $\gamma\ge0$ satisfies
\begin{equation}\label{eq:mars-bound}
\E\sk(\nabla f(x_U))
\le \Theta_T+\frac{2\bmax\gamma b\eta^2R^3}{\amin}
+\frac{2\bmax|1-\gamma|L\eta R^2}{\amin}
+2R\sqrt{W_T^\gamma}.
\end{equation}
\end{proposition}
\begin{proof}
With $e_t=m_t-\nabla f(x_t)$,
\[
e_t=\beta_t e_{t-1}+\beta_t r_t^\gamma+w_t^\gamma,
\quad r_t^\gamma=\gamma\mu_t-\Cs t,
\quad w_t^\gamma=\beta_t\gamma(\widehat C_t-\mu_t)+\alpha_t\zeta_t.
\]
The predictable drift satisfies
\[
\norm{r_t^\gamma}
\le\gamma\norm{\mu_t-\Cs t}+|1-\gamma|\norm{\Cs t}
\le\gamma b\eta^2R^2+|1-\gamma|L\eta R.
\]
Lemma~\ref{lem:identities} gives $\Et w_t^\gamma=0$ and
\[
\E\norm{w_t^\gamma}^2
\le\alpha_t^2\sigma^2+2\alpha_t\beta_t\gamma\sigma v\eta R
+\beta_t^2\gamma^2\left[v^2\eta^2R^2+(q^{-1}-1)\E\norm{C_t-B_t}^2\right].
\]
For arbitrary deterministic $\alpha_t$, these estimates can be substituted into Lemma~\ref{lem:tracking}. For schedule~\eqref{eq:alpha-schedule},~\eqref{eq:tail-weights} bounds the average drift contribution by
$\bmax(\gamma b\eta^2R^2+|1-\gamma|L\eta R)/\amin$. Equations~\eqref{eq:initial-weights}--\eqref{eq:pointwise-correction-weights} applied to the noise bound repeat the calculation of Theorem~\ref{thm:lmo} with $v$ replaced by $\gamma v$ and the residual second moment multiplied by $\gamma^2$, so the initial-error and martingale contributions are at most $2\sigma/T$ and $\sqrt{W_T^\gamma}+2\sqrt2\sigma/\sqrt T$. The descent statement of Lemma~\ref{lem:tracking} gives~\eqref{eq:mars-bound}.
\end{proof}

\paragraph{Bias--variance tradeoff.}
Scaling by $\gamma\in[0,1]$ reduces correction variance and introduces a first-order transport bias. For a zero predictor and gradient-difference transport,
\begin{equation}\label{eq:scaled-storm-mse}
\Et\norm{\gamma J_tC_t/q-\Cs t}^2
=\frac{\gamma^2}{q}\Et\norm{C_t-\Cs t}^2
+\left[(1-\gamma)^2+\frac{\gamma^2(1-q)}q\right]\norm{\Cs t}^2,
\end{equation}
by decomposing the second moment into centered variance and squared bias and applying Lemma~\ref{lem:identities} with $\mu_t=\Cs t$. The bias $(\gamma-1)\Cs t$ enters through the drift term of~\eqref{eq:mars-bound}, consistent with the analysis of \citet{shulgin2026mars}.

\begin{corollary}[Explicit bounds for standard scaling]\label{cor:mars-rates}
Under Proposition~\ref{prop:mars-scaling}, suppose $\cL_T\le\bar\kappa^2\eta^2R^2$ with $\bar\kappa\ge0$ independent of $T$. Fix $q\in(0,1]$, $\ce>0$, and $0<\ca\le1$, and write
$\Omega_q(\bar\kappa)=v^2+(q^{-1}-1)\bar\kappa^2$. The following bounds hold for every integer $T\ge1$ under the stated conditions.
\begin{enumerate}
\item Suppose $0\le\gamma_T\le\bar\gamma$ and $|1-\gamma_T|\le c_\gamma T^{-1/3}$ for constants $\bar\gamma,c_\gamma\ge0$. With $\eta=\ce T^{-2/3}$ and $\amin=\ca T^{-2/3}$,
\begin{align}
\E\sk(\nabla f(x_U))\le{}&
\left[\frac\Delta\ce+2R\sqrt{\ca\sigma^2+\bar\gamma\sigma v\ce R+
\frac{\bar\gamma^2\ce^2R^2}{\ca}\Omega_q(\bar\kappa)}
+\frac{2c_\gamma L\ce R^2}{\ca}\right]T^{-1/3}\notag\\
&+4\sqrt2R\sigma T^{-1/2}
+\left(\frac{L\ce R^2}{2}+\frac{2\bar\gamma b\ce^2R^3}{\ca}\right)T^{-2/3}
+4R\sigma T^{-1}.\label{eq:standard-explicit-fast}
\end{align}
\item For any fixed $\gamma\ge0$, with $\eta=\ce T^{-3/4}$ and $\amin=\ca T^{-1/2}$,
\begin{align}
\E\sk(\nabla f(x_U))\le{}&
\left[\frac\Delta\ce+2R\sigma\sqrt\ca+
\frac{2|1-\gamma|L\ce R^2}{\ca}\right]T^{-1/4}
+2R\sqrt{\gamma\sigma v\ce R}\,T^{-3/8}\notag\\
&+\left[4\sqrt2R\sigma+2\gamma\ce R^2\sqrt{\frac{\Omega_q(\bar\kappa)}\ca}\right]T^{-1/2}
+\frac{L\ce R^2}{2}T^{-3/4}
+\left[4R\sigma+\frac{2\gamma b\ce^2R^3}{\ca}\right]T^{-1}.\label{eq:standard-explicit-slow}
\end{align}
\end{enumerate}
\end{corollary}
\begin{proof}
Using $\bmax\le1$ and $\amin(2-\amin)\ge\amin$ in~\eqref{eq:scaling-noise} gives
\[
W_T^\gamma\le\sigma^2\amin+\gamma\sigma v\eta R
+\frac{\gamma^2\eta^2R^2}{\amin}\Omega_q(\bar\kappa).
\]
For the first tuning,
\[
2R\sqrt{W_T^{\gamma_T}}
\le2R\sqrt{\ca\sigma^2+\bar\gamma\sigma v\ce R+
\frac{\bar\gamma^2\ce^2R^2}{\ca}\Omega_q(\bar\kappa)}\,T^{-1/3},
\]
and the two drift terms in~\eqref{eq:mars-bound} are at most
$2\bar\gamma b\ce^2R^3T^{-2/3}/\ca$ and $2c_\gamma L\ce R^2T^{-1/3}/\ca$.
For the second tuning, splitting the square root into its three nonnegative components gives
\[
2R\sqrt{W_T^\gamma}
\le2R\sigma\sqrt\ca\,T^{-1/4}
+2R\sqrt{\gamma\sigma v\ce R}\,T^{-3/8}
+2\gamma\ce R^2\sqrt{\frac{\Omega_q(\bar\kappa)}\ca}\,T^{-1/2},
\]
and the drift terms are at most $2\gamma b\ce^2R^3T^{-1}/\ca$ and
$2|1-\gamma|L\ce R^2T^{-1/4}/\ca$.
Adding the four terms of $\Theta_T$ proves both bounds.
\end{proof}
The second tuning gives a convergent bound for fixed $\gamma\ne1$; the first applies when the scale approaches one at the stated rate.

\subsection{Scaling Only the Prediction Residual}
Consider instead
\begin{equation}\label{eq:residual-scaled}
\widetilde C_t=B_t+\gamma\frac{J_t}{q}(C_t-B_t),
\qquad m_t=\beta_t(m_{t-1}+\widetilde C_t)+\alpha_t g_t,
\end{equation}
and assume in this subsection that $B_t$ is $\F_{t-1}$-measurable. Since
$\widetilde C_t=\gamma\widehat C_t+(1-\gamma)B_t$, its conditional mean and centered part are
\begin{align}
\widetilde\mu_t&=\gamma\mu_t+(1-\gamma)B_t,\label{eq:residual-scaled-mean}\\
\widetilde C_t-\widetilde\mu_t&=\gamma(\widehat C_t-\mu_t).\label{eq:residual-scaled-centered}
\end{align}
Consequently,
\[
\Et\norm{\widetilde C_t-\widetilde\mu_t}^2
=\gamma^2\left[\Et\norm{C_t-\mu_t}^2+(q^{-1}-1)\Et\norm{C_t-B_t}^2\right],
\]
its covariance with $\zeta_t$ is $\gamma\Et\ip{\zeta_t}{C_t-\mu_t}$, and its transport bias is
\begin{equation}\label{eq:residual-scaled-drift}
\widetilde\mu_t-\Cs t=(\mu_t-\Cs t)+(\gamma-1)(\mu_t-B_t).
\end{equation}
The additional bias is proportional to the mean-prediction error.

\begin{proposition}[LMO bound with residual scaling]\label{prop:residual-scaling}
Under Assumptions~\ref{ass:objective}--\ref{ass:query} with a predictable predictor, define
\[
\mathcal P_T=\frac1T\sum_{t=1}^{T-1}\E\norm{\mu_t-B_t}.
\]
For~\eqref{eq:residual-scaled} with schedule~\eqref{eq:alpha-schedule} and constant $\gamma\ge0$,
\begin{equation}\label{eq:residual-scaled-bound}
\E\sk(\nabla f(x_U))
\le \Theta_T+\frac{2\bmax b\eta^2R^3}{\amin}
+\frac{2\bmax R|1-\gamma|}{\amin}\mathcal P_T
+2R\sqrt{W_T^\gamma}.
\end{equation}
\end{proposition}
\begin{proof}
The tracking recursion has drift~\eqref{eq:residual-scaled-drift} and the same centered noise $w_t^\gamma$ as in Proposition~\ref{prop:mars-scaling}, so the initial and martingale terms have the same bounds. For $\mathcal D_t=\sum_{i=1}^t p_{i,t}\beta_i(\widetilde\mu_i-\Cs i)$, exchanging the finite sums and using~\eqref{eq:tail-weights},
\begin{align*}
\frac1T\sum_{t=0}^{T-1}\E\norm{\mathcal D_t}
\le\frac{\bmax}{\amin T}\sum_{i=1}^{T-1}\E\norm{\widetilde\mu_i-\Cs i}
\le\frac{\bmax}{\amin}\left[b\eta^2R^2+|1-\gamma|\mathcal P_T\right].
\end{align*}
The descent statement of Lemma~\ref{lem:tracking} concludes. For arbitrary deterministic $\alpha_t$, the same lemma applies before the tail-weight simplification.
\end{proof}

\begin{corollary}[Fixed residual scaling with accurate mean prediction]\label{cor:residual-scaling-rate}
Under Proposition~\ref{prop:residual-scaling}, suppose $\cL_T\le\bar\kappa^2\eta^2R^2$ and $\mathcal P_T\le b_B\eta^2R^2$ with $\bar\kappa,b_B\ge0$ independent of $T$. Fix $\gamma\ge0$, $q\in(0,1]$, $\ce>0$, and $0<\ca\le1$. For every integer $T\ge1$, schedule~\eqref{eq:alpha-schedule} with $\eta=\ce T^{-2/3}$ and $\amin=\ca T^{-2/3}$ gives
\begin{align}
\E\sk(\nabla f(x_U))\le{}&
\left[\frac\Delta\ce+2R\sqrt{\ca\sigma^2+\gamma\sigma v\ce R+
\frac{\gamma^2\ce^2R^2}{\ca}\Omega_q(\bar\kappa)}\right]T^{-1/3}\notag\\
&+4\sqrt2R\sigma T^{-1/2}
+\left[\frac{L\ce R^2}{2}+\frac{2(b+|1-\gamma|b_B)\ce^2R^3}{\ca}\right]T^{-2/3}
+4R\sigma T^{-1}.\label{eq:residual-explicit-fast}
\end{align}
A sufficient mean-prediction condition is $\norm{\mu_t-B_t}\le b_B\norm{s_t}^2$ almost surely at every correction index.
\end{corollary}
\begin{proof}
The total drift contribution in~\eqref{eq:residual-scaled-bound} satisfies
\[
\frac{2\bmax b\eta^2R^3}{\amin}
+\frac{2\bmax R|1-\gamma|}{\amin}\mathcal P_T
\le\frac{2(b+|1-\gamma|b_B)\ce^2R^3}{\ca}T^{-2/3},
\]
and
\[
2R\sqrt{W_T^\gamma}
\le2R\sqrt{\ca\sigma^2+\gamma\sigma v\ce R+
\frac{\gamma^2\ce^2R^2}{\ca}\Omega_q(\bar\kappa)}\,T^{-1/3}.
\]
Adding the four terms of $\Theta_T$ proves~\eqref{eq:residual-explicit-fast}. The sufficient condition gives
$\mathcal P_T\le b_B\eta^2R^2(T-1)/T\le b_B\eta^2R^2$, since $\norm{s_t}\le\eta R$.
\end{proof}
A deterministic quadratic with known diagonal Hessian admits an exact diagonal predictor and satisfies the mean-prediction condition with $b_B=0$. In general, Corollary~\ref{cor:residual-scaling-rate} requires a bound on $\mathcal P_T$ in addition to the regret-controlled quantity $\cL_T$. By~\eqref{eq:irreducible} and Cauchy--Schwarz,
\[
\mathcal P_T\le\sqrt{\frac{T-1}{T}\cdot\frac1T\sum_{t=1}^{T-1}\E\norm{\mu_t-B_t}^2}
\le\sqrt{\cL_T},
\]
so $\cL_T\le\bar\kappa^2\eta^2R^2$ implies $\mathcal P_T\le\bar\kappa\eta R$. The complete $T^{-1/4}$ bound is then~\eqref{eq:scaling-complete-slow} with $b_*=b$ and $\ell_*=|1-\gamma|\bar\kappa$.

\subsection{Complete Rates for Both Scaling Variants}\label{app:scaling-explicit}
Let $\cL_T\le\bar\kappa^2\eta^2R^2$, write
\[
\mathsf A_q^\gamma(\bar\kappa)=\frac\Delta\ce+2R\sqrt{\ca\sigma^2+\gamma\sigma v\ce R+
\frac{\gamma^2\ce^2R^2}{\ca}\Omega_q(\bar\kappa)},
\]
and keep $\mathsf U=4\sqrt2R\sigma$ and $\mathsf I=4R\sigma$ from~\eqref{eq:rate-constants}. Suppose the average transport drift satisfies
\[
\frac1T\sum_{t=1}^{T-1}\E\norm{r_t^\gamma}
\le b_*\eta^2R^2+\ell_*\eta R,
\]
where $r_t^\gamma$ is the drift of the variant. The drift estimates above give the following choices.
\begin{center}
\begin{tabular}{lll}
\toprule
Variant & $b_*$ & $\ell_*$\\
\midrule
Standard scaling & $\gamma b$ & $|1-\gamma|L$\\
Residual scaling, $\mathcal P_T\le b_B\eta^2R^2$ & $b+|1-\gamma|b_B$ & $0$\\
Residual scaling, $\mathcal P_T\le c_P\eta R$ & $b$ & $|1-\gamma|c_P$\\
\bottomrule
\end{tabular}
\end{center}

For $\eta=\ce T^{-2/3}$ and $\amin=\ca T^{-2/3}$, Propositions~\ref{prop:mars-scaling} and~\ref{prop:residual-scaling} imply
\begin{align}
\E\sk(\nabla f(x_U))\le{}&\mathsf A_q^\gamma(\bar\kappa)T^{-1/3}
+\mathsf U T^{-1/2}
+\left(\frac{L\ce R^2}{2}+\frac{2b_*\ce^2R^3}{\ca}\right)T^{-2/3}
+\mathsf I T^{-1}+\frac{2\ell_*\ce R^2}{\ca}.\label{eq:scaling-complete-fast}
\end{align}
The last term does not decay with $T$; it vanishes when $\ell_*=0$. If $0\le\gamma_T\le\bar\gamma$ and $|1-\gamma_T|\le c_\gamma T^{-1/3}$, standard scaling satisfies
\begin{align}
\E\sk(\nabla f(x_U))\le{}&
\left[\mathsf A_q^{\bar\gamma}(\bar\kappa)+\frac{2c_\gamma L\ce R^2}{\ca}\right]T^{-1/3}
+\mathsf U T^{-1/2}
+\left(\frac{L\ce R^2}{2}+\frac{2\bar\gamma b\ce^2R^3}{\ca}\right)T^{-2/3}
+\mathsf I T^{-1},\label{eq:scaling-near-one}
\end{align}
which is~\eqref{eq:standard-explicit-fast} in this notation. For accurate mean prediction, the second row of the table gives $\ell_*=0$ and recovers~\eqref{eq:residual-explicit-fast}.

For the slower tuning $\eta=\ce T^{-3/4}$ and $\amin=\ca T^{-1/2}$,
\begin{align}
\E\sk(\nabla f(x_U))\le{}&
\left(\frac\Delta\ce+2R\sigma\sqrt\ca+\frac{2\ell_*\ce R^2}{\ca}\right)T^{-1/4}
+2R\sqrt{\gamma\sigma v\ce R}\,T^{-3/8}
+\left(\mathsf U+2\gamma\ce R^2\sqrt{\frac{\Omega_q(\bar\kappa)}\ca}\right)T^{-1/2}\notag\\
&+\frac{L\ce R^2}{2}T^{-3/4}
+\left(\mathsf I+\frac{2b_*\ce^2R^3}{\ca}\right)T^{-1}.\label{eq:scaling-complete-slow}
\end{align}
Both formulas follow from $\bmax\le1$ and $\amin(2-\amin)\ge\amin$ in the finite-horizon propositions and substitution of the chosen powers of $T$; for~\eqref{eq:scaling-complete-slow}, the square root of the three noise terms is split, giving the exponents $1/4$, $3/8$, and $1/2$, and the drift terms give $2b_*\ce^2R^3T^{-1}/\ca$ and $2\ell_*\ce R^2T^{-1/4}/\ca$.

\paragraph{Learned predictor.}
With the diagonal predictor, use~\eqref{eq:explicit-LT} in place of a fixed residual bound. The additional contribution to $W_T^\gamma$ is at most
\[
\frac{\gamma^2\eta^2R^2(q^{-1}-1)\Gamma_q}{\amin\sqrt T},
\]
and multiplying its square root by $2R$ gives $\gamma\mathsf G_qT^{-7/12}$ for the fast tuning and $\gamma\mathsf G_qT^{-3/4}$ for the slow tuning, with $\mathsf G_q$ from~\eqref{eq:Gq}. In the fast tuning $\mathsf A_q^\gamma(\kappa_T)$ replaces $\mathsf A_q^\gamma(\bar\kappa)$, and in the slow tuning $\Omega_q(\kappa_T)$ replaces $\Omega_q(\bar\kappa)$. In the near-one bound the additional coefficient is $\bar\gamma\mathsf G_q$. Any mean-prediction condition in the table remains an additional hypothesis. For a prescribed accuracy, use a uniform bound $\kappa_T\le\kappa$; $\kappa=V_C$ always qualifies. For~\eqref{eq:scaling-near-one}, and for~\eqref{eq:scaling-complete-fast} with $\ell_*=0$, the horizon~\eqref{eq:learned-horizon} applies with the displayed coefficients. For~\eqref{eq:scaling-complete-slow}, use $T_0=\max\{2,(\mathsf A/\varepsilon)^4\}$ and $T_\varepsilon=\lceil T_0[1+\mathcal E(T_0)/\varepsilon]^4\rceil$, where $\mathsf A$ is the $T^{-1/4}$ coefficient and $\mathcal E$ collects every remaining term; the same power comparison proves sufficiency.

\subsection{Oracle Counts and Implementation}
Both variants use $T$ fresh gradients and $T-1$ Bernoulli decisions. Their gradient-difference cost is $\E N_T=(1+q)T-q$, and HVP transport uses $q(T-1)$ HVPs in expectation in addition to $T$ gradients. At $\gamma=0$, standard scaling needs no correction queries, while residual scaling needs them only to train future predictors.

On an unqueried step, standard and residual-scaled transport equal $\gamma B_t$ and $B_t$; on a queried step they equal
$\gamma[B_t+(C_t-B_t)/q]$ and $B_t+\gamma(C_t-B_t)/q$. The prediction is formed before the current correction is queried or fitted, and the predictor is trained against the unscaled target $C_t$. The variants agree when $B_t=0$ or $\gamma=1$. At $q=1$ and $\gamma\ne1$, the predictor cancels only for standard scaling.

\section{Additional Scope}\label{app:scope}
Predictable query probabilities $q_t>0$ preserve the one-step identities with $q$ replaced by $q_t$. The cumulative term then involves $\sum_t\E[(q_t^{-1}-1)\norm{C_t-B_t}^2]$ rather than a constant multiple of $\cL_T$, and its interaction with adaptive sampling, cost, and online learning requires separate treatment. The rates and binomial counts in this paper are stated for fixed $q$.

A bounded update controls $\norm{s_t}$ but does not by itself give the alignment inequality~\eqref{eq:lmo-alignment}; our results are for exact LMOs and the inexact models of Appendix~\ref{app:inexact}. Clipping the final transport estimator changes its conditional mean, whereas clipping the predictor coefficients in~\eqref{eq:online-update} preserves Lemma~\ref{lem:identities}.

\section{Experimental Details}\label{app:mnist1d}
\subsection{MNIST1D}
\paragraph{Data and model.}
We use the frozen MNIST1D dataset of \citet{greydanus2024mnist1d} ($4000$ training and $1000$ test sequences of length $40$, ten classes), with inputs standardized by training-set statistics. The model is a multilayer perceptron with layer widths $40$, $100$, $100$, $10$, GELU activations, He initialization, and the cross-entropy loss; it has $d=15210$ parameters. GELU makes the network smooth in its parameters and permits HVP evaluation. The theory additionally assumes uniform smoothness, a Lipschitz Hessian, and oracle moment bounds; we do not verify these global assumptions for this unconstrained network.

\paragraph{Algorithm and accounting.}
All runs use Algorithm~1 with the Euclidean LMO $u_t=m_t/\norm{m_t}$, $\gamma=1$, the zero predictor unless stated, $m_0$ equal to one stochastic gradient, and schedule~\eqref{eq:alpha-schedule} with constant $\eta$. Gradient-difference queries evaluate the current batch at $x_{t-1}$; HVP queries compute $\nabla^2f_{\xi_t}(x_t)s_t$ jointly with the fresh gradient by forward-over-reverse differentiation. Skipped corrections are not evaluated. On the single CPU core used for all runs, a joint gradient and HVP costs $2.4$ gradients, so one HVP query counts as $c_H=1.4$ gradients. A run stops at the first multiple of $100$ iterations at which its realized work, counting every evaluated gradient and correction, reaches the budget: $16000$ gradient-equivalents at batch size $4$ and $8000$ at batch size $32$. The realized work exceeds the budget by at most $2.0\%$ across all runs, and for $q<1$ the number of iterations varies across seeds.

\paragraph{Tuning and evaluation.}
For each method, every pair $(\eta,\amin)$ in the grid of Section~\ref{sec:experiments} is run on seeds $0$ and $1$, and the pair whose worse final training loss over the two seeds is smallest is selected. The selected pair is evaluated on seeds $10$ to $19$, which set the initialization, the batches, and the query decisions. Batches are drawn uniformly with replacement, and the three sources of randomness use separate random streams, so runs with the same seed share their initialization and batch sequence. If an evaluation seed diverged, the next pair would be used; this did not occur in the reported runs. Ratios to $q=1$ are computed per seed, and confidence intervals use the $t$ distribution with $9$ degrees of freedom.

\paragraph{Estimating $\hat\omega$.}
Along a reference run at batch size $4$ ($\eta=0.05$, $\amin=0.1$, without correction), at the checkpoints of Table~\ref{tab:omega} we take the next step $s$, the full-batch gradient difference $\Cs{}$, and $400$ independent batches to estimate $\E\norm{C-\Cs{}}^2$ for the gradient-difference correction; $\hat\omega=\norm{\Cs{}}^2/\E\norm{C-\Cs{}}^2$. The estimated $\delta$ scales as $b^{-1/2}$ in the batch size $b$.

\begin{table}[h]
\caption{Estimated $\hat\omega$ along the reference run, by batch size.}
\label{tab:omega}
\centering
\small
\begin{tabular}{rrcccc}
\toprule
Step & Training loss & $b=1$ & $b=4$ & $b=32$ & $b=256$\\
\midrule
$0$ & $3.11$ & $0.088$ & $0.33$ & $2.6$ & $21$\\
$500$ & $1.58$ & $0.037$ & $0.15$ & $1.3$ & $11$\\
$2000$ & $1.03$ & $0.017$ & $0.060$ & $0.46$ & $3.6$\\
$8000$ & $0.38$ & $0.004$ & $0.020$ & $0.17$ & $1.4$\\
\bottomrule
\end{tabular}
\end{table}

\begin{table}[h]
\caption{MNIST1D: final training loss at $q=1$ and ratio of the final loss at $q<1$ to it, with paired $95\%$ confidence intervals over $10$ seeds, at equal oracle cost ($16000$ gradient-equivalents at batch size $4$, $8000$ at batch size $32$).}
\label{tab:mnist1d}
\centering
\small
\setlength{\tabcolsep}{3pt}
\begin{tabular}{llccc}
\toprule
Batch & Oracle & $q=1$ & $q=1/2$ & $q=1/4$\\
\midrule
$4$ & STORM & $0.425$ & $0.80\,[0.75,0.86]$ & $0.70\,[0.67,0.73]$\\
$4$ & HVP & $0.470$ & $0.78\,[0.70,0.87]$ & $0.66\,[0.60,0.71]$\\
$32$ & STORM & $0.077$ & $0.99\,[0.69,1.29]$ & $0.82\,[0.74,0.90]$\\
$32$ & HVP & $0.089$ & $0.73\,[0.66,0.81]$ & $0.57\,[0.53,0.61]$\\
\bottomrule
\end{tabular}
\end{table}

\begin{table}[h]
\caption{MNIST1D: test accuracy at $q=1$ (in $\%$, mean over $10$ seeds) and paired difference at $q<1$ from it (percentage points, $95\%$ interval), at equal oracle cost; settings selected by training loss.}
\label{tab:mnist1d-acc}
\centering
\small
\setlength{\tabcolsep}{3pt}
\begin{tabular}{llccc}
\toprule
Batch & Oracle & $q=1$ & $q=1/2$ & $q=1/4$\\
\midrule
$4$ & STORM & $61.9$ & $+1.55\,[+0.43,+2.67]$ & $+1.31\,[+0.08,+2.54]$\\
$4$ & HVP & $62.7$ & $+0.40\,[-1.24,+2.04]$ & $+2.28\,[+1.27,+3.29]$\\
$32$ & STORM & $62.9$ & $+0.18\,[-1.11,+1.47]$ & $-1.75\,[-2.73,-0.77]$\\
$32$ & HVP & $62.2$ & $+0.90\,[-0.28,+2.08]$ & $+1.29\,[+0.58,+2.00]$\\
\bottomrule
\end{tabular}
\end{table}

\begin{table}[h]
\caption{MNIST1D, diagonal predictor with the settings of the zero-predictor run: test accuracy (mean and standard deviation over $10$ seeds, in $\%$), paired difference from the zero predictor at the same $q$ (percentage points, $95\%$ interval), and residual ratio $\sum_tJ_t\norm{C_t-B_t}^2/\sum_tJ_t\norm{C_t}^2$ (mean over seeds).}
\label{tab:mnist1d-diag}
\centering
\small
\begin{tabular}{lllccc}
\toprule
Batch & Method & $q$ & Test accuracy & Difference from zero predictor & Residual ratio\\
\midrule
$4$ & STORM & $1/2$ & $64.01\pm1.61$ & $+0.58\,[-0.68,+1.84]$ & $0.995$\\
$4$ & STORM & $1/4$ & $63.18\pm1.21$ & $-0.01\,[-0.18,+0.16]$ & $0.989$\\
$4$ & HVP & $1/2$ & $63.61\pm1.32$ & $+0.54\,[-0.31,+1.39]$ & $0.995$\\
$4$ & HVP & $1/4$ & $65.06\pm1.56$ & $+0.11\,[-1.64,+1.86]$ & $0.996$\\
\midrule
$32$ & STORM & $1/2$ & $63.06\pm1.85$ & $-0.01\,[-1.35,+1.33]$ & $0.983$\\
$32$ & STORM & $1/4$ & $61.37\pm0.79$ & $+0.23\,[-0.41,+0.87]$ & $0.954$\\
$32$ & HVP & $1/2$ & $63.21\pm1.24$ & $+0.14\,[-0.79,+1.07]$ & $0.985$\\
$32$ & HVP & $1/4$ & $63.99\pm1.23$ & $+0.53\,[-0.45,+1.51]$ & $0.984$\\
\bottomrule
\end{tabular}
\end{table}

\begin{table}[h]
\caption{All MNIST1D runs: final training loss (mean and standard deviation over $10$ seeds), test accuracy, number of iterations (mean over seeds), and selected $(\eta,\amin)$.}
\label{tab:mnist1d-full}
\centering
\small
\begin{tabular}{llcccc}
\toprule
Batch & Method & Training loss & Test acc. & Iterations & $(\eta,\amin)$\\
\midrule
$4$ & STORM, $q=1$ & $0.425\pm0.025$ & $0.619$ & $8000$ & $(0.1,0.3)$\\
$4$ & STORM, $q=1/2$ & $0.341\pm0.029$ & $0.634$ & $10710$ & $(0.1,0.3)$\\
$4$ & STORM, $q=1/4$ & $0.297\pm0.016$ & $0.632$ & $12840$ & $(0.05,0.3)$\\
$4$ & HVP, $q=1$ & $0.470\pm0.032$ & $0.627$ & $6700$ & $(0.1,0.3)$\\
$4$ & HVP, $q=1/2$ & $0.365\pm0.031$ & $0.631$ & $9450$ & $(0.1,0.3)$\\
$4$ & HVP, $q=1/4$ & $0.308\pm0.035$ & $0.650$ & $11910$ & $(0.1,0.3)$\\
\midrule
$32$ & STORM, $q=1$ & $0.077\pm0.009$ & $0.629$ & $4000$ & $(0.1,0.3)$\\
$32$ & STORM, $q=1/2$ & $0.075\pm0.033$ & $0.631$ & $5380$ & $(0.1,0.3)$\\
$32$ & STORM, $q=1/4$ & $0.062\pm0.005$ & $0.611$ & $6440$ & $(0.05,0.3)$\\
$32$ & HVP, $q=1$ & $0.089\pm0.006$ & $0.622$ & $3400$ & $(0.1,0.3)$\\
$32$ & HVP, $q=1/2$ & $0.066\pm0.013$ & $0.631$ & $4750$ & $(0.1,0.3)$\\
$32$ & HVP, $q=1/4$ & $0.051\pm0.007$ & $0.635$ & $5950$ & $(0.1,0.3)$\\
\bottomrule
\end{tabular}
\end{table}

The diagonal predictor uses $\eta_h=0.1$ and $\hmax=20$, without further tuning, together with the frozen $(\eta,\amin)$ of the zero-predictor run at the same $q$ and the same seeds, as for CIFAR-10. Table~\ref{tab:mnist1d-diag} reports the results: every accuracy interval contains zero, and the predictor removes $0.4\%$ to $1.1\%$ of the importance-weighted residual at batch size $4$ and $1.5\%$ to $4.6\%$ at batch size $32$, where $\hat\omega$ is larger. All twelve selected floors equal the cap $\amin=0.3$. The reported values are those of the final iterate of runs stopped at a realized budget, so the experiments do not instantiate the finite-horizon guarantees, which concern a uniformly sampled iterate at a fixed horizon. At the iteration count of the $q=1$ run, the ratio of the training loss at $q<1$ to that at $q=1$ is $1.01$ and $1.24$ for STORM at $q=1/2$ and $1/4$ and $1.06$ and $1.08$ for HVP at batch size $4$, and $1.26$, $2.94$, $1.24$, and $1.47$ at batch size $32$, each run being tuned for its own horizon. Test accuracies lie between $61\%$ and $65\%$, since the model overfits the $4000$ training sequences, and the settings were selected by training loss. Table~\ref{tab:mnist1d-acc} gives the paired accuracy differences. They agree in sign with the loss ratios of Table~\ref{tab:mnist1d}, except for STORM with $q=1/4$ at batch size $32$, whose lower training loss comes with a lower test accuracy. In the batch-size-$32$ runs of Figure~\ref{fig:mnist1d-bs32}, curves can cross because $\eta$ is tuned for the final loss of each method separately.

\begin{figure}[h]
\centering
\includegraphics[width=0.5\textwidth]{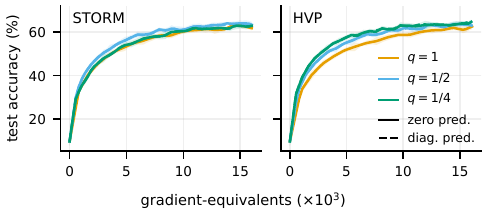}
\caption{MNIST1D, batch size $4$. Test accuracy against gradient-equivalent evaluations for gradient-difference (left) and HVP (right) transport at $q\in\{1,1/2,1/4\}$, with the zero predictor (solid; mean and standard deviation over $10$ seeds) and the diagonal predictor (dashed; mean). Figure~\ref{fig:mnist1d-bs32} shows batch size $32$.}
\label{fig:mnist1d}
\end{figure}

\begin{figure}[h]
\centering
\includegraphics[width=0.5\textwidth]{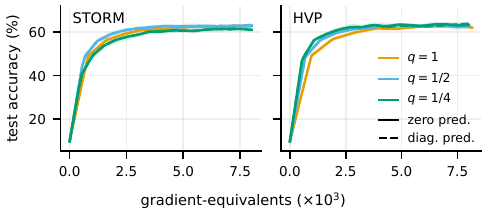}
\caption{MNIST1D, batch size $32$; same setting as Figure~\ref{fig:mnist1d}.}
\label{fig:mnist1d-bs32}
\end{figure}

\subsection{Controlled Quadratic}\label{app:synth}
\paragraph{Problem.}
The eigenvalues of $A$ are $\lambda_i=10^{-2+2(i-1)/(d-1)}$ for $i=1,\ldots,d$, and $x^*$ is a fixed draw from $\mathcal N(0,I)$; $x_0=0$. The oracle of Section~\ref{sec:experiments} is unbiased and satisfies~\eqref{eq:similarity} with modulus $\delta$, since $\E[E_\xi]=0$ and $\E\norm{E_\xi s}^2=d\,\delta^2\,\E\norm{u}^2\,\E(v^\top s)^2=\delta^2\norm{s}^2$. Since $E_\xi$ is not symmetric, $g(\cdot;\xi)$ is not the gradient of a scalar sample loss; the analysis uses only the oracle properties above. The gradient variance at $x$ is $\delta^2\norm{x-x^*}^2+1$. Along a normalized run, $\norm{x_t-x^*}\le\norm{x_0-x^*}+\eta t$, so Assumption~\ref{ass:objective} holds over each finite horizon $T$ with $\sigma_T^2=\delta^2(\norm{x_0-x^*}+\eta T)^2+1$; the experiment illustrates the mechanism under this state-dependent noise and does not verify a horizon-uniform $\sigma$. On this problem the gradient difference on a common sample equals $(A+E_\xi)s_t$, the product of $s_t$ with an unbiased, nonsymmetric estimator of the Hessian, so the two transports coincide.

\paragraph{Protocol.}
Algorithm~1 uses the Euclidean LMO, $\gamma=1$, and schedule~\eqref{eq:alpha-schedule}. For $q\in\{1,1/2,1/4\}$ the horizon is $T=\mathrm{round}((4000+q)/(1+q))$, so that $(1+q)T-q\approx4000$. The implementation also evaluates an unused final gradient and correction, so its expected work $1+(1+q)T$ lies between $4001$ and $4001.5$ gradients. The grids are $\eta\in\{0.01,0.02,0.05,0.1,0.2\}$, $\amin\in\{0.003,0.01,0.03,0.1,0.3\}$, and, for the diagonal predictor with $\hmax=2$, $\eta_h\in\{0.03,0.3\}$. Selection uses the worse of seeds $0$ and $1$, and evaluation uses seeds $10$ to $19$, as for MNIST1D. The value $\hat\omega$ of a run is the average of $\norm{As_t}^2/(\delta^2\norm{s_t}^2)$ over its steps, which equals $\norm{\Cs t}^2/\E_t\norm{C_t-\Cs t}^2$.

\begin{table}[h]
\caption{Controlled quadratic: ratio of the run-average $\norm{\nabla f(x_t)}$ at $q$ to its value at $q=1$, with paired $95\%$ confidence intervals over $10$ seeds, and the residual ratio of the diagonal predictor at $q=1/2$.}
\label{tab:synth}
\centering
\small
\begin{tabular}{cccccc}
\toprule
& & \multicolumn{2}{c}{Zero predictor} & Diagonal & \\
$\delta$ & $\hat\omega$ ($q=1$) & $q=1/2$ & $q=1/4$ & $q=1/2$ & Residual ratio\\
\midrule
$10$ & $0.0011$ & $0.92\,[0.91,0.94]$ & $0.86\,[0.84,0.88]$ & $0.92\,[0.91,0.94]$ & $1.000$\\
$3$ & $0.013$ & $0.90\,[0.89,0.91]$ & $0.83\,[0.82,0.84]$ & $0.90\,[0.89,0.91]$ & $0.993$\\
$1$ & $0.11$ & $0.91\,[0.90,0.92]$ & $0.83\,[0.82,0.85]$ & $0.91\,[0.90,0.92]$ & $0.933$\\
$0.3$ & $1.1$ & $0.93\,[0.93,0.93]$ & $0.89\,[0.88,0.89]$ & $0.93\,[0.92,0.93]$ & $0.489$\\
$0.2$ & $0.91$ & $1.00\,[0.99,1.00]$ & $0.95\,[0.95,0.96]$ & $0.99\,[0.98,1.00]$ & $0.521$\\
$0.15$ & $1.7$ & $1.06\,[1.05,1.06]$ & $1.01\,[1.01,1.02]$ & $0.98\,[0.97,0.99]$ & $0.447$\\
$0.1$ & $58$ & $1.17\,[1.17,1.18]$ & $1.14\,[1.13,1.14]$ & $1.02\,[1.02,1.03]$ & $0.278$\\
$0.03$ & $988$ & $1.41\,[1.40,1.42]$ & $1.37\,[1.37,1.38]$ & $0.98\,[0.97,0.99]$ & $0.005$\\
\bottomrule
\end{tabular}
\end{table}

In terms of the global constant, the sign change occurs at $L^2/\delta^2$ between $11$ and $44$, because normalized steps move mostly along low-curvature directions. The value $\hat\omega$ depends on the tuned trajectory and is therefore not monotone in $\delta$; at $\delta=0.2$ the selected floor is $\amin=0.03$, and $\hat\omega=0.9$ is below its value $1.15$ at $\delta=0.3$. The residual ratio stays slightly above $1/(1+\hat\omega)$ at large $\hat\omega$ ($0.005$ against $0.001$ at $\delta=0.03$), which reflects the learning error of the predictor. The study takes about four minutes on one CPU core.

\subsection{Unnormalized STORM with the Tuning of Theorem~\ref{thm:unnormalized}}\label{app:prop5}
This experiment uses the method of Theorem~\ref{thm:unnormalized}: the step $x_{t+1}=x_t-\eta m_t$, the zero predictor, and the prescribed parameters $c=c_q$, $\amin=c_q\eta^2$, and $\eta=\min\{\eta_{\max},(\Delta/(K_qT))^{1/3}\}$ of~\eqref{eq:optimal-c}--\eqref{eq:cap-threshold}, with no tuning. The objective is the quadratic of Appendix~\ref{app:synth}, with $L=1$, and the oracle is
\[
g(x;\epsilon,\zeta)=Ar+\delta a\,\epsilon\odot\sin(r/a)+\zeta,\qquad r=x-x^*,\quad a=d^{-1/2},
\]
with $\epsilon\sim\mathcal N(0,I)$ and an independent $\zeta$ with mean zero and $\E\norm{\zeta}^2=1$, the sine acting coordinatewise. It is the gradient of the sample loss $\frac12r^\top Ar-\delta a^2\sum_i\epsilon_i\cos(r_i/a)+\zeta^\top r$ and is unbiased. Its variance $\delta^2a^2\sum_i\sin^2(r_i/a)+1$ is at most $\sigma^2=\delta^2+1$ at every $x$, and since $a|\sin(u/a)-\sin(v/a)|\le|u-v|$, it satisfies~\eqref{eq:similarity} with modulus $\delta$. The constants $L$, $\delta$, and $\sigma$ used in the tuning therefore hold globally. We take $\delta\in\{0.25,0.5,1,2,4\}$, so that $\omega=L^2/\delta^2$ ranges from $16$ to $0.0625$, and $x_0$ at distance $1$ from $x^*$. For budgets $N\in\{2000,8000,32000\}$ the horizon is $T=\mathrm{round}((N+q)/(1+q))$, and the metric is the run average of $\norm{\nabla f(x_t)}^2$, the quantity bounded in~\eqref{eq:unnormalized}. In all $45$ settings its mean over $10$ seeds lies below the bound~\eqref{eq:unnormalized}.

At equal expected budget $N$, the horizons satisfy $T_q\approx N/(1+q)$, and the leading term $3(\Delta^2K_q)^{1/3}T^{-2/3}$ of~\eqref{eq:leading-rate} gives the ratio
\[
\left(\frac{K_q}{K_1}\right)^{1/3}\left(\frac{1+q}2\right)^{2/3}=\cQ(q;\omega)^{2/3}
\]
of the leading squared-stationarity certificates; $\cQ$ itself compares the budgets needed to reach a common accuracy. Table~\ref{tab:prop5} compares the measured fixed-budget ratios with $\cQ^{2/3}$.

\begin{table}[h]
\caption{Unnormalized STORM with the tuning of Theorem~\ref{thm:unnormalized} and the variance-bounded oracle: ratio of the run-average $\norm{\nabla f(x_t)}^2$ at $q$ to its value at $q=1$ (mean over $10$ paired seeds) at budgets $2000$, $8000$, and $32000$, next to the fixed-budget benchmark $\cQ(q;\omega)^{2/3}$.}
\label{tab:prop5}
\centering
\small
\begin{tabular}{cc|cccc|cccc}
\toprule
& & \multicolumn{4}{c|}{$q=1/2$} & \multicolumn{4}{c}{$q=1/4$}\\
$\delta$ & $\omega$ & $\cQ^{2/3}$ & $2000$ & $8000$ & $32000$ & $\cQ^{2/3}$ & $2000$ & $8000$ & $32000$\\
\midrule
$0.25$ & $16$ & $1.569$ & $1.137$ & $1.160$ & $1.211$ & $1.874$ & $1.236$ & $1.273$ & $1.356$\\
$0.5$ & $4$ & $1.187$ & $1.051$ & $1.052$ & $1.070$ & $1.347$ & $1.121$ & $1.128$ & $1.163$\\
$1$ & $1$ & $1.016$ & $0.981$ & $0.977$ & $0.982$ & $1.091$ & $1.022$ & $1.020$ & $1.034$\\
$2$ & $0.25$ & $0.958$ & $0.949$ & $0.948$ & $0.951$ & $0.996$ & $0.975$ & $0.976$ & $0.982$\\
$4$ & $0.0625$ & $0.942$ & $0.949$ & $0.947$ & $0.944$ & $0.968$ & $0.971$ & $0.969$ & $0.965$\\
\bottomrule
\end{tabular}
\end{table}

The direction of the measured change agrees with $\cQ(q;\omega)^{2/3}-1$, which has the sign of $\cQ-1$, in nine of the ten $(\omega,q)$ settings at all budgets; the exception is $\omega=1$ and $q=1/2$, with benchmark $1.016$ and measured ratio $0.98$. For $\omega<1$ the measured ratios lie within $2.1\%$ of $\cQ^{2/3}$, and for $\omega>1$ they are smaller than $\cQ^{2/3}$. This is a qualitative comparison of sufficient bounds: neither the benchmark nor its finite-budget analogue bounds the ratio of the actual errors, and the leading ratio omits the cap and lower-order terms. The $95\%$ confidence half-widths of the measured ratios are at most $0.034$. The implementation evaluates $m_0$, a fresh gradient at each of the $T$ steps, and a correction with probability $q$ at each step, so its expected work is $1+(1+q)T$.

\subsection{Symmetric Curvature Noise}\label{app:symmetric}
We repeat the sweep of Appendix~\ref{app:synth} with $E_\xi=\delta\,\mathrm{diag}(\epsilon)$ and $\epsilon\sim\mathcal N(0,I)$. Then $f_\xi(x)=\frac12(x-x^*)^\top(A+E_\xi)(x-x^*)+\zeta_\xi^\top(x-x^*)$ is a sample loss with gradient $g(x;\xi)$ and Hessian $A+E_\xi$, and $\E\norm{E_\xi s}^2=\delta^2\norm{s}^2$ as before. The protocol, grids, and seeds are those of Appendix~\ref{app:synth}. As there, the gradient variance grows with $\norm{x-x^*}$, so these runs illustrate the mechanism under state-dependent noise rather than verify a rate with a uniform $\sigma$.

\begin{table}[h]
\caption{Symmetric curvature noise: ratio of the run-average $\norm{\nabla f(x_t)}$ at $q$ to its value at $q=1$, with paired $95\%$ confidence intervals over $10$ seeds, and the residual ratio of the diagonal predictor at $q=1/2$.}
\label{tab:symmetric}
\centering
\small
\begin{tabular}{cccccc}
\toprule
& & \multicolumn{2}{c}{Zero predictor} & Diagonal & \\
$\delta$ & $\hat\omega$ ($q=1$) & $q=1/2$ & $q=1/4$ & $q=1/2$ & Residual ratio\\
\midrule
$10$ & $0.00023$ & $0.79\,[0.60,0.98]$ & $0.78\,[0.57,0.99]$ & $0.77\,[0.60,0.94]$ & $1.012$\\
$3$ & $0.0017$ & $0.95\,[0.83,1.08]$ & $0.93\,[0.78,1.07]$ & $0.96\,[0.84,1.07]$ & $1.029$\\
$1$ & $0.02$ & $0.92\,[0.89,0.96]$ & $0.87\,[0.84,0.89]$ & $0.93\,[0.90,0.97]$ & $1.349$\\
$0.3$ & $0.21$ & $1.04\,[0.99,1.09]$ & $0.99\,[0.96,1.02]$ & $1.01\,[0.95,1.08]$ & $1.157$\\
$0.15$ & $14$ & $1.10\,[1.06,1.14]$ & $1.08\,[1.05,1.12]$ & $0.97\,[0.95,0.99]$ & $0.638$\\
$0.1$ & $59$ & $1.19\,[1.17,1.21]$ & $1.17\,[1.14,1.19]$ & $1.01\,[1.00,1.03]$ & $0.363$\\
$0.05$ & $358$ & $1.32\,[1.27,1.38]$ & $1.30\,[1.24,1.35]$ & $0.99\,[0.94,1.03]$ & $0.011$\\
$0.03$ & $1003$ & $1.43\,[1.40,1.46]$ & $1.40\,[1.38,1.43]$ & $0.98\,[0.94,1.03]$ & $0.006$\\
$0.01$ & $9058$ & $1.49\,[1.48,1.50]$ & $1.46\,[1.45,1.47]$ & $0.91\,[0.90,0.92]$ & $0.004$\\
\bottomrule
\end{tabular}
\end{table}

The pattern of Table~\ref{tab:synth} persists: $q<1$ lowers the metric at large noise ($0.87$ at $\delta=1$ and $q=1/4$), raises it at small noise (up to $1.49$), and the diagonal predictor restores parity with $q=1$ at small noise. Two differences remain. The selected floor $\amin$ changes between $\delta=0.3$ and $\delta=0.15$, and $\hat\omega$ on the $q=1$ runs jumps from $0.21$ to $14$, so the crossover is located only within this interval. At small $\hat\omega$ the residual ratio of the diagonal predictor exceeds one ($1.35$ at $\delta=1$): with noise in the same coordinates as the predictor, the online iterates fluctuate, which is consistent with a learning-error cost of the kind bounded by the regret term of Theorem~\ref{thm:regret}, although $\eta_h$ is selected for the stationarity metric rather than set to the theoretical value. The stationarity metric shows no resolved change ($0.92$ without and $0.93$ with the predictor at $\delta=1$). The confidence intervals at $\delta\ge3$ are wide because the selected $(\eta,\amin)$ at $q=1$ differs from that at $q<1$.

\subsection{CIFAR-10}\label{app:cifar}
\paragraph{Model and data.}
The network is a randomly initialized ResNet-18 \citep{he2016resnet} with $10$ outputs, a $3\times3$ input convolution with stride $1$ and no max pooling, GELU activations, and group normalization with $32$ groups in place of batch normalization, so that correction queries do not change running statistics. GELU makes the network smooth in its parameters and permits HVP evaluation; as for MNIST1D, the global constants of the theory are not verified. The $50000$ CIFAR-10 training images are split, with a fixed seed, into $45000$ for training and $5000$ for validation; the $10000$ test images are used only for the final evaluation. Training uses random crops with padding $4$ and horizontal flips, batch size $128$, and drops the incomplete final batch, so that one epoch has $351$ updates. Batches are formed by shuffling the training set each epoch, so within an epoch a batch is not conditionally unbiased for the full training objective, and that part of Assumption~\ref{ass:objective} does not hold exactly. Initialization, shuffling, and query decisions use separate random generators.

\paragraph{Algorithm and accounting.}
All runs use Algorithm~1 with the Euclidean LMO over all parameters, $\gamma=1$, $m_0$ equal to one stochastic gradient, schedule~\eqref{eq:alpha-schedule} with constant $\eta$, float32 arithmetic, and no weight decay, clipping, or learning-rate schedule. The STORM correction evaluates the gradient at the previous parameters on the same augmented batch; the HVP correction is $\nabla^2\ell(\theta_t;\xi_t)\Delta_t$ at the current parameters, computed by forward-mode differentiation of the gradient. A fresh gradient costs one gradient-equivalent, a STORM query one more, and an HVP query $c_H=2.32$ more, the incremental cost measured on an A100 GPU. Accounting uses realized queries, and a run stops before the update that would exceed $35100$ gradient-equivalents; the realized work lies between $35097.7$ and $35100$.

\paragraph{Tuning and evaluation.}
For each method and $q$, $(\eta,\amin)$ was selected with the zero predictor in two stages: a pilot at budgets $128$, $512$, and $2048$ over $\eta\in\{0.001,0.003,0.01,0.03,0.1,0.3,1\}$ and $\amin\in\{0.03,0.1,0.3\}$, followed by full-budget runs of the pilot choice and its neighbors ($\eta$ divided or multiplied by $3$, larger floors), with up to two expansions of the step-size range. Each candidate ran on seeds $0$ and $1$, and the candidate with the smaller worse-seed validation cross-entropy was selected. For STORM at $q=1/2$ the selected step size lies at the upper end of the tested range. The selected settings were frozen and evaluated on seeds $10$ to $14$; test accuracy is that of the final model. Table~\ref{tab:cifar} reports test accuracy, the paired difference from $q=1$ with $95\%$ Student-$t$ intervals, and the mean training time per run.

\begin{table}[h]
\caption{CIFAR-10: selected $(\eta,\amin)$, mean number of updates within the budget, test accuracy (mean and standard deviation over $5$ seeds, in $\%$), paired difference from $q=1$ (percentage points, $95\%$ interval), and training time in seconds.}
\label{tab:cifar}
\centering
\small
\begin{tabular}{llcccccc}
\toprule
Method & $q$ & $\eta$ & $\amin$ & Updates & Test accuracy & Difference from $q=1$ & Time\\
\midrule
STORM & $1$ & $0.27$ & $0.3$ & $17550$ & $87.34\pm0.70$ & baseline & $1053$\\
STORM & $1/2$ & $0.27$ & $0.3$ & $23376$ & $88.05\pm1.13$ & $+0.72\,[-0.01,+1.44]$ & $1063$\\
STORM & $1/4$ & $0.03$ & $0.3$ & $28091$ & $81.64\pm1.12$ & $-5.69\,[-7.12,-4.26]$ & $1105$\\
\midrule
HVP & $1$ & $0.09$ & $0.1$ & $10570$ & $85.99\pm0.38$ & baseline & $934$\\
HVP & $1/2$ & $0.09$ & $0.3$ & $16218$ & $86.95\pm0.22$ & $+0.96\,[+0.64,+1.28]$ & $966$\\
HVP & $1/4$ & $0.09$ & $0.3$ & $22258$ & $87.50\pm0.22$ & $+1.52\,[+0.96,+2.07]$ & $1046$\\
\bottomrule
\end{tabular}
\end{table}

For STORM, $q=1/2$ has a higher mean accuracy than $q=1$, with the paired $95\%$ interval including zero, and $q=1/4$ is clearly worse. This ordering agrees qualitatively with the leading coefficient of Proposition~\ref{prop:ratio}, which is larger at $q=1/4$ than at $q=1/2$ for every $\omega$, by at least $0.032$; at $q=1/4$ the selected step size is also nine times smaller than at $q=1$. For HVP transport both $q<1$ improve on $q=1$, with the larger gain at $q=1/4$. At equal oracle work the randomized runs take longer in wall time, since they perform more updates: $1046$ against $934$ seconds for HVP at $q=1/4$ and $q=1$, and $1105$ against $1053$ seconds for STORM. The comparison concerns oracle work, not wall time.

\paragraph{Diagonal predictor.}
The diagonal predictor, with $\eta_h=0.1$ and $\hmax=20$ and no further tuning, was run at $q\in\{1/2,1/4\}$ with the frozen $(\eta,\amin)$ of the corresponding zero-predictor run and the same seeds. Table~\ref{tab:cifar-diag} compares it with the zero predictor. All test-accuracy intervals contain zero, and the predictor removes $0.4\%$ to $0.9\%$ of the importance-weighted residual. By~\eqref{eq:irreducible}, a predictable predictor can remove at most the mean part of the residual; we did not estimate $\hat\omega$ on this problem.

\begin{table}[h]
\caption{CIFAR-10, diagonal predictor: test accuracy (mean and standard deviation over $5$ seeds, in $\%$), paired difference from the zero predictor at the same $q$ (percentage points, $95\%$ interval), and residual ratio $\sum_tJ_t\norm{C_t-B_t}^2/\sum_tJ_t\norm{C_t}^2$ (mean over seeds).}
\label{tab:cifar-diag}
\centering
\small
\begin{tabular}{llccc}
\toprule
Method & $q$ & Test accuracy & Difference from zero predictor & Residual ratio\\
\midrule
STORM & $1/2$ & $88.29\pm0.33$ & $+0.24\,[-1.10,+1.58]$ & $0.992$\\
STORM & $1/4$ & $80.99\pm0.88$ & $-0.65\,[-1.33,+0.02]$ & $0.991$\\
HVP & $1/2$ & $86.93\pm0.32$ & $-0.01\,[-0.50,+0.48]$ & $0.996$\\
HVP & $1/4$ & $87.51\pm0.47$ & $+0.00\,[-0.47,+0.47]$ & $0.995$\\
\bottomrule
\end{tabular}
\end{table}

\paragraph{Compute.}
Each MNIST1D method takes about one minute on one CPU core, tuning included, and the MNIST1D study about fifteen minutes; the quadratic studies take a few minutes on one CPU core. Each CIFAR-10 run takes $15$ to $21$ minutes on one A100 GPU. The code for all studies will be released with the final version, and we can provide it to reviewers during the review process on request.


\end{document}